\documentclass[sigconf]{aamas}

\setcopyright{none}              
\renewcommand\footnotetextcopyrightpermission[1]{} 

\usepackage{balance} 

\usepackage{comment}
\usepackage{amsthm}
\usepackage{algorithm}
\usepackage{algpseudocode}
\usepackage{tikz}
\usepackage{subcaption}

\usepackage{graphicx}
\graphicspath{ {} }

\newtheorem{theorem}{Theorem}
\newtheorem{lemma}{Lemma}
\newtheorem{definition}{Definition}

\newtheorem{corollary}{Corollary}
\newtheorem{example}{Example}
\newtheorem{remark}{Remark}

\newcommand{\mycomment}[1]{\noindent [#1]}
\newcommand{\fb}[1]{\mycomment{{\color{red}{FB:~#1}}}}
\newcommand{\ed}[1]{\mycomment{{\color{green}{ED:~#1}}}}

\newcommand{\distr}[1]{\mathbb D(#1)}

\DeclareMathOperator*{\argmin}{arg\,min}

\newcommand{\last}[1]{\text{last}\left(#1\right)}

\newcommand{\paths}[1]{\text{Paths}\left(#1\right)}
\newcommand{\expect}[4]{\mathbb E^{#3}_{#1,#2}\left(#4\right)}

\newcommand{\Augment}[1]{\mbox{Augment}}

\newcommand{\clip}[3]{\mbox{clip}\left({#1,#2,#3}\right)}

\acmSubmissionID{\#1022}

\title[\textsc{ProSh}: Probabilistic  Shielding\\ for Model-free Reinforcement Learning]{\textsc{ProSh}: Probabilistic  Shielding\\ for Model-free Reinforcement Learning}

\author{Edwin Hamel-De le Court$^*$}
\affiliation{
  \institution{Imperial College}
  \city{London}
  \country{United Kingdom}}
\email{e.hamel-de-le-court@imperial.ac.uk}

\author{Gaspard Ohlmann$^*$}
\affiliation{
  \institution{}
  \city{Mulhouse}
  \country{France}}
\email{gaspard.ohlmann@outlook.com}

\author{Francesco Belardinelli}
\affiliation{
  \institution{Imperial College}
  \city{London}
  \country{United Kingdom}}
\email{francesco.belardinelli@imperial.ac.uk}

\begin{abstract}
Safety is a major concern in reinforcement learning (RL): we aim at developing RL systems that not only perform optimally, but are also safe to deploy by providing formal guarantees about their safety. To this end, we introduce Probabilistic Shielding via Risk Augmentation (\textsc{ProSh}), a model-free algorithm for safe reinforcement learning under cost constraints. \textsc{ProSh} augments the Constrained MDP state space with a risk budget and enforces safety by applying a shield to the agent's policy distribution using a learned cost critic. The shield ensures that all sampled actions remain safe in expectation. We also show that optimality is preserved when the environment is deterministic. Since \textsc{ProSh} is model-free, safety during training depends on the knowledge we have acquired about the environment. We provide a tight upper-bound on the cost in expectation, depending only on the backup-critic accuracy, that is \emph{always} satisfied during training. Under mild, practically achievable assumptions, \textsc{ProSh} guarantees safety even at training time, as shown in the experiments.
\end{abstract}


\newcommand{\BibTeX}{\rm B\kern-.05em{\sc i\kern-.025em b}\kern-.08em\TeX}

\makeatletter
\def\@makefnmark{\hbox{\textsuperscript{\normalsize\@thefnmark}}}
\makeatother

\begin{document}


\pagestyle{fancy}
\fancyhead{}


\maketitle 

$*$ \emph{These authors contributed equally to this work.}

\section{Introduction}

A key challenge in AI is designing agents that learn to act optimally in unknown environments \cite{sutton2018reinforcement}. Reinforcement Learning (RL) -- particularly when combined with deep neural networks -- has made impressive strides in domains such as games \cite{DBLP:journals/corr/MnihKSGAWR13}, robotics \cite{SurveyRobotics}, and autonomous driving \cite{KHJMRALBS19}. However, ensuring safety during both training and deployment remains a critical barrier to real-world adoption. This has led to a growing interest in Safe RL, where agents must optimize rewards under constraints, e.g., budgeted costs or formal safety requirements.
This paper focuses on safe learning in Constrained Markov Decision Processes (CMDPs), where agents maximize expected discounted reward, while keeping cumulative expected cost below a given threshold. Traditional methods, such as Lagrangian approaches \cite{Achiam2019BenchmarkingSE,Stooke2020ResponsiveSF,Ray2019BenchmarkingSC}, allow to converge toward a safe policy, but do not provide formal guarantees on safety, neither during training nor on the provided policy.

To tackle this problem, we introduce \textsc{ProSh}: a novel probabilistic shielding method that ensures safety also at training time, in model-free RL. Unlike classic shielding methods that restrict unsafe actions based on a  model of the environment, \textsc{ProSh} operates on distributions, augments the CMDP with a risk budget, and uses a learned cost critic to guide safe exploration. Crucially, \textsc{ProSh} does not assume access to a simulator or environment abstraction, and is compatible with continuous environments.


We also show that optimizing over shielded policies in the augmented space is sufficient for constrained optimality in the base CMDP in the deterministic setting. We implement \textsc{ProSh} as an off-policy deep RL algorithm and evaluate it on standard Safe RL benchmarks. Notably, \textsc{ProSh} turns out to be safe during training in the experiments.
To summarize, the contributions of the paper are as follows.
\begin{enumerate}
    \item We introduce \textsc{ProSh}, a shielding method for CMDP that assumes no knowledge of the environment's dynamics.
    \item We provide formal safety guaranties during training, depending only on the accuracy of the learned cost critic.
    \item We show that optimizing across shielded policies in the Risk-Augmented MDP suffices for constrained optimality (in the deterministic case).
    \item We implement \textsc{ProSh} as a DRL algorithm and evaluate its optimality and safety on relevant benchmarks. In particular, \textsc{ProSh} leads to significantly less cost violations in expectation, even in early training.
\end{enumerate}

Due to space restrictions, the proofs of all the results appear in the supplementary material, which also includes the code for the experiments.


\paragraph{Related Work.} 
We refer to \cite{DBLP:journals/corr/abs-2205-10330} for a survey on Safe Reinforcement Learning, while here we focus on the works most closely related to our contribution. 
Policy-based methods are arguably the most popular approaches to Safe RL. They usually consist of extending known RL algorithms (e.g., PPO \cite{SPPO}, TRPO \cite{DBLP:journals/corr/SchulmanLMJA15}, or SAC \cite{HSAC}) with safety constraints, either by using a Lagrangian approach \cite{Achiam2019BenchmarkingSE, Stooke2020ResponsiveSF}, changing the objective function \cite{DBLP:conf/aaai/LiuDL20}, or by modifying the update process \cite{DBLP:journals/corr/abs-2002-06506}.
Several of these algorithms have become widely used for benchmarking against newer Safe RL methods, and are implemented in state-of-the-art Safe RL frameworks \cite{Achiam2019BenchmarkingSE,ji2023safetygymnasiumunifiedsafereinforcement,ji2023omnisafeinfrastructureacceleratingsafe}.
%

\textbf{Shielding}
restricts the agent’s actions during training and deployment to ensure safety \cite{ABENTShielding,DBLP:conf/concur/0001KJSB20,DBLP:conf/atal/Elsayed-AlyBAET21,DBLP:conf/ijcai/YangMRR23}. Introduced in \cite{ABENTShielding} using LTL safety formulas, shielding has been extended to probabilistic settings in \cite{DBLP:conf/concur/0001KJSB20}, although without formal guaranties. Formal safety under probabilistic constraints has been proved in \cite{HBG25}. More recent work allows shielding without prior models: \cite{SEHTF20} uses a learned safety critic to restrict actions, while \cite{GB24} combines shielding with a learned dynamics model. A survey on shielding methods is given in \cite{10039301}. Compared to prior approaches, we focus on the general case of model-free, continuous environments with probabilistic constraints, while prior approaches typically make further assumptions.


\textbf{Lagrangian approaches}
\cite{Achiam2019BenchmarkingSE,Stooke2020ResponsiveSF,Ray2019BenchmarkingSC} convert CMDPs into unconstrained problems by introducing a dual variable to penalize cost violations in the reward objective. These methods are effective in both discrete and continuous action spaces and are compatible with policy gradient algorithms (e.g. TRPO \citep{DBLP:journals/corr/SchulmanLMJA15}, PPO \citep{SPPO}). Given enough iterations and tuning, they provide good tradeoffs between cost violations and reward, but cannot guaranty safety at any time, which limits their use in safety-critical scenarios.

\textbf{State augmentation techniques}
with parameters representing how far the agent is from being unsafe have also been studied. In \cite{DBLP:journals/corr/abs-2102-11941}, the MDP is state-augmented with Lagrange multipliers. In Saute-RL \cite{DBLP:conf/icml/SootlaCJWMWA22}, the states are augmented with the remaining safety budget, which is
used to reshape the objective. We also use an augmentation approach, but we not only augment states but actions too, allowing agents more freedom in the future probabilistic repartition of the remaining safety budget.

\textbf{Q-learning} 
is a key model-free RL algorithm for discrete environments \citep{watkins1992q}, known to converge in the tabular case under mild assumptions. However, extending it to function approximation —especially with neural networks—introduces challenges \citep{Mnih2015HumanlevelCT}. Stabilization techniques such as double Q-learning \citep{hasselt2010double}, dueling networks \citep{Wang2016DuelingNA}, and distributional methods \citep{Bellemare2017ADP} help address instability. Classic Q-learning is  prone to overestimation bias \citep{thrun1993issues,hasselt2010double}. To mitigate these issues, we adopt a TD3-like method \citep{fujimoto2018td3}.
Throughout, we assume the backup $Q$-value approximates a fixed point of the Bellman operator—an assumption we discuss in the paper.

\textbf{Trust-region approaches} such as CPO \citep{Achiam2019BenchmarkingSE}, FOCOPS \cite{zhang2020first}, and Safe Policy Iteration \citep{satija2020constrained} are On-policy algorithms that rely on providing local safety bounds that hold within a small neighborhood of the current policy. This is useful to bound constraint violations of the next policy updates if the trust region is small, but the accuracy of the bound depends heavily on the size of the trust region. In addition, these methods do not provide a formal guaranty that the policy update step is safe, even late in the training.

\textbf{Lyapunov-based methods} rely on enforcing constraint satisfaction by projecting updates to satisfy a Lyapunov decrease condition defined relative to a fixed safe baseline policy \citep{chow2018lyapunov,DBLP:journals/corr/abs-1901-10031}. The guarantees provided in \citep{chow2018lyapunov,DBLP:journals/corr/abs-1901-10031} depend on accurate safety critics, local linearizations around the current policy, and a well-behaved backup policy. We rely on restricting the agent's actions to enforce safety, but we use a state-augmentation technique rather than a fixed Lyapunov function, which yields explicit, model-free training-time bounds that depend only on the backup critic’s approximation error. Furthermore, our backup policy is learned jointly with the main actor and used as a fallback within the shield, rather than as a fixed reference that defines the feasible set of actions.
\section{Background}
In this section we provide the background on Reinforcement Learning (RL) and 
the notations that will be used in the rest of the paper. This enables us to define the two key problems that we analyse: the Reinforcement Learning Problem (RLP)~\cite{10.5555/3312046}  and the Constrained RLP
(CRLP)~\cite{achiam2017constrained,garcia2015comprehensive}.

\subsection{Key Concepts}


A \textbf{Markov Decision Processes}
 (MDP) is a tuple $\mathcal
M=\langle S,A,P,s_{i}, r\rangle$, where $S$ is a
set of \textit{states}, with $s_i$ as initial state\footnote{This can be assumed wlog compared to a model with an \emph{initial probability distribution} since it is always possible to add a new initial state to such a model with an action from this initial state whose associated probability distribution is the aforementioned initial probability distribution.}; $A$ is a mapping that associates every state
$s\in S$ to a nonempty finite set $A(s)$ of \emph{actions}; for $\mathcal{SA} = \{ (s,a),~a\in A(s) \}$, $P: \mathcal{SA} \to \distr{S}$ is a
\textit{transition probability function} that maps every state-action
pair $(s,a) \in \mathcal{SA}$ to a probability
distribution over $S$; 
and $r:\mathcal{S}\mathcal{A}\mapsto \mathbb R$ is the
\emph{reward function}. 
An MDP is finite if the sets of
states and actions are finite.
Finally, a \emph{Constrained MDP} (CMDP) 
is an MDP 
additionally equipped with a \emph{cost function} $c:\mathcal{S}\mathcal{A}\to \mathbb{R}_+$.

\textbf{Paths.}
A finite (resp.~infinite) {\em path} 
is a finite
(resp.~infinite) word $\zeta=s_0 a_0 \ldots s_{n-1}a_{n-1}s_n$
(resp.~$\zeta=s_0 a_0\ldots s_n a_n\ldots$), such that for every 
$k\leq n$ (resp.~for every
$k$), $s_{k-}$ is a state in $S$, $a_{k-1}$ is an action in $A(s_{k-1})$, and $s_k$ is in the support of $P(s_{k-1},a_{k-1})$. 

\textbf{Policies.}
A {\em policy} $\pi$ 
is a
mapping that associates every finite path $\zeta$
to an action of the probability distribution $\distr{A(\last{\zeta})}$. A policy  is {\em memoryless} if $\pi(\zeta)$ only depends on $\last{\zeta}$; it is {\em deterministic} if for any finite path $\zeta$,
$\pi(\zeta)$ is a Dirac distribution; it is \emph{flipping} if for any finite path $\zeta$,
$\pi(\zeta)$ is a mixture of two Dirac distributions. 
%
Throughout the paper, we let $
\sum_{a\in A(s)} \lambda_a a$ denote 
the probability distribution corresponding to sampling action $a$ with probability $\lambda_a$. As an example, if $\pi_1$ and $\pi_2$ are two deterministic policies on $\mathcal{M}$, the policy defined for any $s\in \mathcal{S}$ as $\pi(s) = 0.5\pi_1(s) + 0.5 \pi_2(s) $ denotes the flipping policy taking in $s$ the actions $\pi_1(s)$ and $\pi_2(s)$ with equal probability.

For any policy
$\pi$ of $\mathcal M$ and  state $s\in S$, let $\mathcal
M^s_\pi$ be the {\em Markov chain} induced by $\pi$ in $\mathcal M$
starting from state $s$ (c.f.~\cite{baier2008principles}). let $\mathcal M_\pi$ denote
$\mathcal M^{s_{i}}_\pi$ and $P_\pi$ the transition
function of $\mathcal M_\pi$. We denote the usual probability measure
induced by the Markov chain $\mathcal M^s_\pi$ on $\paths{\mathcal M}$
by $\text{prob}^{s}_{\mathcal M,\pi}$. We refer to \cite{baier2008principles,BSSstochastic} for full details.

\textbf{Discounted expectations.}
For any random variable\\
$X:\paths {\mathcal M}\mapsto \mathbb R$, let $\expect{\mathcal
  M}{\pi}{s}{X}$ be the expectation of $X$ w.r.t.~the
probability distribution $\text{prob}^{s}_{\mathcal M,\pi}$, and let
$\expect{\mathcal M}{\pi}{}{X}$ denote\\ $\expect{\mathcal
  M}{\pi}{s_{i}}{X}$. 
Given a path $\zeta = s_0a_0\dots s_n a_n \dots$,
we denote the {\em discounted cumulative reward} along $\zeta$ as $\mathcal{R}(\zeta) = \sum_{t\in\mathbb N} \gamma^t_r r(s_t,a_t)$, and its {\em discounted cumulative cost} as  $\mathcal{C}(\zeta) = \sum_{t\in \mathbb{N}} \gamma^t_c c(s_t,a_t)$ , where
$\gamma_r$ and $\gamma_c$ are the discount factors for reward $r$ and cost $c$ respectively. For any policy $\pi$ if a CMDP $\mathcal{M}$, we denote the \textbf{discounted cumulative reward} of $\pi$ as
\[
\mathcal{R}_\mathcal{M}^s(\pi) = \mathop{\mathbb{E}}_{\zeta \sim \pi} \mathcal{R}^s(\zeta) = \mathop{\mathbb{E}}_{(s_t,a_t)_t\sim \pi,s_0=s} \sum_{t\in \mathbb{N}} \gamma^t_r r(s_t,a_t),
\]
and the \textbf{discounted cumulative cost} of $\pi$ as 
\[
\mathcal{C}^s_\mathcal{M}(\pi)  = \mathop{\mathbb{E}}_{\zeta \sim \pi} \mathcal{C}^s(\zeta) = \mathop{\mathbb{E}}_{(s_t,a_t)_t \sim \pi,s_0=s} \sum_{t\in \mathbb{N}} \gamma_c^t c(s_t,a_t).
\]

We let $\mathcal R_{\mathcal M}(\pi)=\mathcal R^{s_i}_{\mathcal M}(\pi)$ and $\mathcal C_{\mathcal M}(\pi)=\mathcal C^{s_i}_{\mathcal M}(\pi)$, and omit $\mathcal M$ when there is no ambiguity.

\textbf{Q-values} 
are functions \( Q: \mathcal{S}\times\mathcal{A} \to \mathbb{R} \) that estimates the expected cost of taking action \( a \) in state \( s \), and following some policy thereafter. In our context, we are especially interested in Q-values approximating the \emph{minimal discounted cost}, defined as
\[
Q_b^*(s,a) := \min_{\pi:\pi(s) = a} \mathcal{C}_{\mathcal{M}}^s(\pi).
\]
We denote by \( Q_b \) any approximation of \( Q_b^* \), and say it is an {\em \( \varepsilon \)-approximation} if \( \|Q_b - Q_b^*\|_\infty \leq \varepsilon \).

\subsection{Reinforcement Learning Problems}

This work is motivated by two central optimization problems in RL. The first, classic Reinforcement Learning problem (RLP), focuses on optimizing the expected reward~\cite{sutton2018reinforcement}. The second, constrained RLP (CRLP), introduces explicit safety requirements in the form of cost constraints~\cite{garcia2015comprehensive,achiam2017constrained}. Our method \textsc{ProSh} is designed specifically for CRLP, and aims to provide provably safe probabilistic guarantees even in the absence of a known model.

\begin{definition}[RL Problems] \label{RL_problems}
    Let $\mathcal{M}$ be an MDP, $\gamma_r$ the reward discount factor, and 
    $\Pi$ the set of policies over $\mathcal{M}$. 
    \[
(RLP): 
    \text{Find policy $\pi^*\in \Pi$ such that }
    \mathcal{R}(\pi^*)= \max_{\pi \in \Pi} \mathcal{R}(\pi)
    \]
%
%
%
Further, 
let
$\mathcal{M}'$ be a CMDP, $\gamma_r$ and $\gamma_c$ the discount factors for rewards and costs respectively, and $d>0$ be a cost threshold . 

Denote 
    $\Pi_{\leq d}$ the set of policies $\pi$ such that $\mathcal{C}(\pi)\leq d$. 
    %
    \[
    \text{ (CRLP) :
    Find policy $\pi^*\in \Pi_{\leq d}$  such that }  
    \mathcal{R}(\pi^*) = \max_{\pi \in \Pi_{\leq d}} \mathcal{R}(\pi)
    \]
    
\end{definition}
Notice that Definition~\ref{RL_problems} is not restricted to memoryless policies, as is usually the case \cite{achiam2017constrained, DBLP:journals/corr/abs-2002-06506, Achiam2019BenchmarkingSE}. This is because 
in the case of two different discount factors for the reward and the cost, optimal memoryless policies for CRLP are not guaranteed to exist \cite{Altman1999ConstrainedMD}. Furthermore, in practical implementations, 
the cost discount factor is often higher than the reward discount factor since constraint satisfaction is paramount in CRLP. The additional difficulty does not break the optimality of the method we provide in the deterministic case.

\paragraph{Shielding and State-Augmentation challenges: a motivation example.}


We illustrate the limitations of traditional shielding techniques through the simple CMDP $\mathcal M_1$ shown in Fig.~\ref{ex-1}, with $\gamma_r = \gamma_c = 1$ and cost threshold $d = 0.5$. The goal is to solve the constrained RL problem (CRLP) for $\mathcal M_1$:
%

%
\begin{figure}[H]
\centering
\begin{tikzpicture}[->, >=stealth, shorten >=1pt, semithick,
    state/.style={circle, draw, minimum size=20pt, inner sep=1pt},
    action/.style={rectangle, draw, minimum size=16pt, inner sep=2pt}
]

\node[state]    (s0) at (0,0)   {$s_0$};

\node[action]   (a0) at (1,  0.8) {$a_0$};
\node[state]    (s1) at (3.4,  0.8) {$s_1$};
\node[action]   (a1) at (1, -0.8) {$a_1$};
\node[state]    (s2) at (3.4, -0.8) {$s_2$};

\node[action]   (a2) at (5.1,  0.8) {$a_2$};
\node[action]   (a3) at (5.1, -0.8) {$a_3$};

\node[state]    (s4) at (6.8, 0)   {$s_4$};  

\path (s0) edge (a0);
\path (a0) edge node[pos=0.55, yshift=8pt] {$r=1$} 
                  node[pos=0.55, yshift=-8pt] {$c=1$} (s1);

\path (s0) edge (a1);
\path (a1) edge node[pos=0.55, yshift=8pt] {$r=0$} 
                  node[pos=0.55, yshift=-8pt] {$c=0$} (s2);

\path (s1) edge (a2);
\path (a2) edge node[pos=0.55, yshift=20pt, xshift=10pt] {$r=0$} 
                  node[pos=0.55, yshift=8pt,xshift=10pt] {$c=0$} (s4);

\path (s2) edge (a3);
\path (a3) edge node[pos=0.55, yshift=-8pt,xshift=10pt] {$r=0$} 
                  node[pos=0.55, yshift=-20pt,xshift=10pt] {$c=0$} (s4);

\end{tikzpicture}
\caption{CMDP $\mathcal M_1$ with cost threshold $d = 0.5$ and $\gamma_r= \gamma_c= 1$.}
\label{ex-1}
\end{figure}
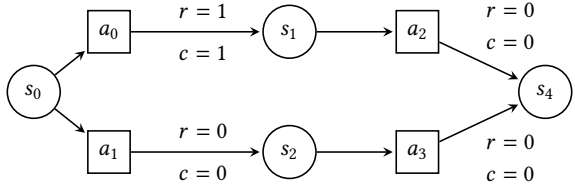

The optimal safe policy $\pi^*$ takes action $a_0$ and $a_1$ with equal probability, yielding a cost $\mathcal{C}(\pi^*)=0.5$ and a reward $\mathcal{R}(\pi^*)=0.5$. This solution cannot be found using hard safety enforcement such as classic shielding~\cite{ABENTShielding}, which would block $a_0$. Nor does a state-augmentation based approach such as \cite{DBLP:conf/icml/SootlaCJWMWA22}, as path $s_0 a_0s_1a_2s_4$ has a cumulative cost above the threshold. We claim that {\em achieving safe optimality requires reasoning over expectations}.

\section{\textsc{ProSh}: Theoretical Foundations}\label{sec-PSRA}

We now present the theoretical backbone of \text{ProSh}, a scalable shielding mechanism for CRLP that provides formal safety guarantees. %
Our approach uses \emph{risk-augmentation}, similarly to ~\cite{HBG25}, whereby each state 
is paired with a risk budget representing the expected discounted cost that the agent is allowed to incur from that point onward. Unlike traditional shielding approaches, \textsc{ProSh} considers probabilistic expectations that the main actor chooses, and enforces these two structural properties:
\begin{itemize}
    \item \textbf{Approximate consistence.} The combined risk after one step is, in expectancy, approximately equal to the risk of the current step.
    
    \item \textbf{Approximate realizability.} For every augmented state, there exists a policy whose expected cumulative cost is approximately less than or equal to the current risk. 
\end{itemize}

Our goal is to construct policies that satisfy the safety constraint \emph{throughout training and deployment}, even in the absence of a known model.
To this end, we augment the CMDP with a dynamically updated notion of risk.

\subsection{Risk-Augmented CMDPs}

Hereafter, for any $Q$-value $Q_b$ of a CMDP $\mathcal M$, for any state $s$ and action $a\in A(s)$, we let $Q_b(s)$ denote 
$\gamma_c\min_{a\in A(s)}(Q_b(s,a)) $.

\begin{definition}[Risk-augmented CMDP]\label{def-risk-aug}
For any CMDP $\mathcal M=(S,A,s_{i},P,R,C,d,\gamma_r,\gamma_c)$ with non-negative costs, maximal cost $c_{\max} = ||c||_{L^\infty(\mathcal{A}\mathcal{S})}$, and $Q$-value $Q_b$, we define $\overline{\mathcal M}^{Q_b}_d$ as the MDP with $C_{\max} = \frac{c_{max}}{(1-\gamma_c)}$, such that
\begin{itemize}
\item States : $\overline{S}= \left\{(s, x) \mid s \in S, x \in \left[-C_{\max};C_{\max}\right]\right\}$, with initial state $\overline{s}_{i}=(s_{i},d)$;
\item Actions : $\overline{A}(s,x)=\left\{(a,y)\mid a\in A(s), y\in \left[-C_{\max};C_{\max}\right]\right\}$;

\item Rewards : $\overline{r}((s,x),(a,y))=r(s,a)$;
\item Costs : $\overline{c}((s,x),(a,y))=c(s,a)$
\item Transition Probability Function $\overline{P}$: for $\overline{s}=(s,x)$,  $\overline{s}'=(s',x')$ in $\overline{S}$, and $\overline{a}=(a,y)\in \overline{A}(\overline{s})$, we have 
\[\overline{P}\left( \overline{s},\overline{a},\overline{s}' \right) =
    \begin{cases}
        0 &~\text{ if } x'\neq y-Q_b(s,a)+Q_b(s')\\
        P(s,a,s') &\text{ if }x'= y-Q_b(s,a)+Q_b(s')
    \end{cases}
\]
\end{itemize}
\end{definition}

The risk-augmented CMDP introduces a new coordinate for states and actions:  the \emph{risk}. The policy can then choose any risk for its next actions, and the transition function $\overline{P}$ simply ensures that the corresponding risk is spread coherently to each state in the non-deterministic case.
We will introduce later the conditions on the policy to ensure that the second parameter $x$ indeed plays the role of a \emph{running budget}, recording how much cumulative discounted cost the agent is allowed to incur.

We now define a subset of policies on $\overline{\mathcal{M}}$, the {\em valued policies}. 

\begin{definition}[Valued policy]
A policy $\bar\pi$ on the risk-augmented CMDP $\overline{\mathcal{M}}$
is 
\emph{valued} if there exists a mapping
\(
y_{\bar\pi}\colon\overline{\mathcal{S}}\to\bigl(A(s)\rightarrow\mathbb{R}\bigr)
\)
such that, for every augmented state $(s,x)\in\overline{\mathcal{S}}$, we can write with the notation $(a,x)_\delta = \delta_{a,x}$, the Dirac Delta distribution choosing the action $(a,x)$ with probability $1$,
\[
\bar\pi(s,x) 
=\sum_{a\in A(s)} P_{\bar\pi}(a\mid s,x)\;
      \bigl(a,\;y_{\bar\pi}(s,x)(a)\bigr)_\delta
\]
In other words, to any underlying state–risk pair $(s,x)$, the policy
first assigns a risk value $r_{\bar\pi}(s,x)(a)$ to every available
action $a\in A(s)$. 
\end{definition}

We call $y_{\bar\pi}$ a {\em valuation function}, and we let $\overline{\Pi}_{val}$ denote the set of all valued policies.

Asking the agent to attach a risk value to every possible action is not as restrictive as it looks. Indeed, we will show that considering only the set of valued policies is sufficient for optimality.

\subsection{Shielded Policies, Safety and Optimality}\label{sec:theory}

Among valued policies, we focus on those that keep the risk budget
synchronized with the CMDP dynamics. We call those the $Q_b$-shielded policies. In the rest of the paper, for any $Q$-value $Q_b$, we let $\pi_b(Q_b)(s)=\argmin_{a\in A(s)} Q_b(s,a)$, and when there is no ambiguity, we write only $\pi_b$ for $\pi_b(Q_b)$. This action is the safest according to the estimation $Q_b(s,a)$ and will be used to decrease the expected cost when necessary.
%
%
\begin{definition}[$Q_b$-Shielded policies] 
    Let $\mathcal{M}$ be a CMDP, and let $Q_b(\mathcal{M})$ be any $Q$-valued function defined on $\mathcal{A}\mathcal{S}$. Denote $\overline{\mathcal{M}}$ the corresponding augmented CMDP. A policy $\bar \pi \in \overline \Pi_{val}$ is said to be $Q_b$-\emph{shielded} if for any $(s,x)\in \bar{ \mathcal{S}}$, if $x\geq Q_b(s)$, there exist $\lambda$ and $\{y_a\}_{a\in A(s)}$, with $y_a\geq Q_b(s,a)$ for all $a\in A(s)$, such that 
%
\[\label{QbShielded-decomp}
\left\{
\begin{aligned}
    \bar{\pi}(s,x)=&  (1-\lambda)\sum_a P_{\bar \pi}((a,y_a)\mid (s,x)) (a,y_a)_\delta \\
    &+\lambda (\pi_b(s),Q_b(s,\pi_b(s)))_\delta,\\
    x \geq \gamma_c (1&-\lambda) \sum_a P_\pi((a,y_a)\mid (s,x)) y_a + \gamma_c \lambda Q_b(s,\pi_b(s)),
\end{aligned}
\right.
\]
    
    and if $x< Q_b(s)$, there exists $z\leq x$ such that 
\begin{equation*}
    \begin{cases}
    \bar{\pi}(s,x)=(\pi_b(s),z),\\
    x\geq \gamma_c z
    \end{cases}
\end{equation*}
\end{definition}

Intuitively, the policies are shielded when they always allow a budget for the actions larger than the estimated minimal cost of taking this action. Moreover, the total budget allowed for the actions is, in expectancy, smaller than the current budget.

We now define the corresponding Shield-Map, which takes any memoryless valued policy $\bar \pi$ of $\overline{\mathcal{M}}$ and outputs a $Q_b$-shielded policy.

\begin{definition}[Shield-map]\label{def:shield-map}
    Let $\bar \pi$ be a memoryless valued policy of the augmented MDP $\overline{\mathcal{M}}$. We define the {\em shield-map} $\Xi$ 
    such that, for any state $(s,x)$ of $\mathcal M$, 
    if we let $ \tilde{y}_a=\max(\tilde{y}_a,Q_b(s,a))$, and
    \[
    \tilde \pi (s,x) = \sum_{a\in A(s)} P_{\bar \pi}(a\mid s,x) (a,\tilde y_a)_\delta,\quad t=\gamma_c \sum_{a\in A(s)} P_{\bar \pi}(a\mid s,x)\tilde y_a,
    \]
    we have
    \[
    \Xi(\bar \pi)(s,x) = 
    \begin{cases}
        \tilde \pi(s,x) & \text{if }t \leq x\\
        (\pi_b(s),x/\gamma_c)_\delta & \text{if }x<Q_b(s)\\
        (1-\lambda) \tilde \pi(s,x)_\delta + \lambda (\pi_b(s),Q_b(s,\pi_b(s)))_\delta & \text{otherwise,}\\
    \end{cases}
    \]
    where 
    \[
    \lambda=\frac{\gamma_c t-x}{\gamma_c t -\gamma_c Q_b(s,\pi_b(s))}.
    \]
\end{definition}

\noindent
\textit{Intuitively}, the shield map blends the original policy with the estimated safest action $\pi_b$ so that the resulting policy always satisfies the $Q_b$-shielded condition. When the policy already respects this safety constraint, no modification is applied. If estimation errors lead to a state where any distribution would exceed the available budget, the shield falls back to the safest estimated action $\pi_b$ and allocates the entire budget to it.

Because the constraint in the definition mirrors the Bellman equation
for discounted costs, the risk $x$ is indeed an approximate upper bound of the \emph{expected future cost} of the shielded policy that starts from the augmented state $(s,x)$. The next theorem makes this intuition precise.


\begin{theorem}[Safety Bounds for $Q_b$-shields]\label{thm:safety}\label{thm:safe_nondet}
Let $\mathcal{M}$ be a CMDP with cost discount factor $\gamma_c$, and
let $Q_b^{*}$ be its optimal state–action cost function. Assume that $Q_b$ is a $Q$-value, and let $\Delta_b = ||Q_b-Q_b^*||_{L^\infty(\mathcal{S}\mathcal{A})}$.  

For any policy
$\bar\pi\in\overline\Pi_{\mathrm{val}}$ that is $Q_b$-shielded, the expected
discounted cost from the augmented state $(s_0,x_0)$ satisfies
\[
\mathcal{C}^{(s_0,x_0)}(\bar\pi)
\;\le\;
x_0+\frac{2\Delta_b}{1-\gamma_c},
\qquad
\text{whenever } x_0\ge \frac{Q_b(s_0)}{\gamma_c}.
\]

\end{theorem}

Theorem~\ref{thm:safety} shows that an $\varepsilon$-accurate $Q_b$ is enough to keep any $Q_b$-shielded policy inside the budget, up to a small additive slack. Moreover, the policy we consider in the augmented CMDP corresponds to policies in the original CMDP that have the same cost and reward, as stated in the next resilts. 

\begin{definition}[Projection onto the base CMDP]\label{def:projection}
For any $Q_b$-shielded policy $\bar\pi\in\overline\Pi_{\mathrm{val}}$ of $\overline{\mathcal{M}}$, we define the backward projection $T(\bar \pi)$ as the memoryful policy tracking a single scalar variable $m$ as 

\begin{itemize}
    \item In state $s$ and with current memory $m-x$, sample $(a, x') \sim \bar\pi(s, x)$ and execute $a_t$ in $\mathcal{M}$.
    \item As state $s'$ is reached in $\mathcal{M}$, update the memory $m\leftarrow x'-Q_b(s,a)+Q_b(s')$.
\end{itemize}
\end{definition}

We can now prove the following preservation result.
\begin{theorem}[Preservation] \label{thm:projection} 
The cost and reward or $T(\bar \pi)$ in Def.~\ref{def:projection} satisfy
\[
\mathcal{C}_{\mathcal{M}}(T(\pi))
   =\mathcal{C}_{\bar{\mathcal{M}}}(\bar\pi),
\qquad
\mathcal{R}_{\mathcal{M}}(T(\pi))
   =\mathcal{R}_{\bar{\mathcal{M}}}(\bar\pi).
\]

\end{theorem}


Theorem \ref{thm:projection} 
states that a $Q_b$-shielded policy keeps its cost and reward when mapped back to the original CMDP.  This 
allows us to transfer our safety results from $\overline{\mathcal{M}}$ to $\mathcal{M}$.
It remains  to
show that \emph{some} $Q_b$-shielded policy attains the constrained optimum.  The next
theorem precisely answers this question. Recall that flipping policies associate to every state a mixture of two Dirac distributions.

    \newcommand{\rr}{\mathcal{E}}
    
\begin{theorem}[Optimality of the shielded policies]\label{thm:optimality}
    Let $\mathcal{M}$ be a deterministic CMDP with safety threshold $d$ and initial state $s_i$,
    $Q_b$ be a Q-value,
    $\Delta_b=||Q_b-Q_b^*||_{L^\infty(\mathcal{S}\mathcal{A})}$. 
    
    Further, let $\Pi$ the set of all policies of $\mathcal{M}$, $\overline \Pi_{sh}$ be the set of shielded policies of $\overline{\mathcal{M}}$, and $\overline \Pi^f$ be the set of valued flipping policies.
    
    The, we have the following     for $\rr=  \frac{2\Delta_b\gamma_c}{1-\gamma_c}$:

    \[
    \max_{\pi \in \Pi,~\mathcal{C}(\pi)\leq d} \mathcal{R}(\pi) \leq \max_{\bar \pi \in \Xi(\overline \Pi^f),~x_0\leq \gamma_c d+\rr} \mathcal{R}^{(s_i,x_0)}(\bar \pi),
    \]
    \[
    \max_{\bar \pi \in \Xi(\overline \Pi^f),~x_0\leq \gamma_c d+\rr} \mathcal{R}^{(s_i,x_0)}(\bar \pi) \leq \max_{\pi \in \overline \Pi_{sh},~x_0\leq \gamma_c d+\rr} \mathcal{R}^{(s_i,x_0)}(\bar \pi),
    \]
\end{theorem}

We now summarize the 
results obtained in this section, which represent the theoretical backbone of \textsc{ProSh}.

\paragraph{From constrained to unconstrained optimisation.}
Let $\Pi_{val}$ the set of valued policies (potentially unsafe) of the augmented MDP, and $\Pi_{\Xi}=\Xi(\Pi_{val})$, then for $\Delta_b = ||Q_b-Q_b^*||_{L^\infty(\mathcal{S}\mathcal{A})}$, 
\begin{itemize}
    \item The set of shielded policies $\Pi_{\Xi}$, starting with a risk up to $x_0=d+2\Delta_b$ is enough to reach optimality, i.e.,
    \[
    \max_{\bar \pi \in \bar \Pi,~x_0\leq d+2\Delta_b} \mathcal{R}^{(s_i,x_0)}(\Xi(\bar\pi)) \geq \max_{\pi \in \Pi} \mathcal{R}(\pi).
    \]
    \item Any shielded policy $\Xi(\bar \pi)\in \Pi_{\Xi}$ starting with a risk up to $x_0=d+2\Delta_b$ has a cost satisfying
    \[
    \mathcal{C}^{(s_i,x_0)}(\Xi(\bar \pi)) \leq d+2\Delta_b + \frac{\Delta_b}{1-\gamma_c}.
    \]
\end{itemize}

Hence, we can optimize among all policies of the augmented MDP and apply the shield. Alternatively, one can choose a dynamic $x_0 < d$ to ensure safety during training. This can be combined with any approach ensuring $\Delta_b\rightarrow 0$, or making it sufficiently small. 

\section{Safe RL through Probabilistic Shielding}

In this section, we show how the theoretical foundations in Section \ref{sec-PSRA} are used to derive learning algorithms for CRLP. We first present a general approach and derive the corresponding safety guarantees. We then propose our implementation of the \textsc{ProSh} algorithm. 

\subsection{Learning Algorithms}


In this section we present a minimal version of the training loop for the main actor in Algorithm~\ref{alg:prosh-correct}. Then, in Sec.~\ref{sec:ProSh-TD3} we provide a complete implementation. \textsc{ProSh} behaves like "classical" shielding \cite{ABENTShielding}, but the shield acts on distribution rather than actions. The shield $\Xi$ uses the values provided by the critic $Q_b$ and combines it with the backup action to obtain a safe distribution.

\begin{algorithm}[H]
\caption{\textsc{ProSh}: Main Actor only}
\label{alg:prosh-correct}
\begin{algorithmic}[1]
\State \textbf{Input:} cost budget $d$, margin $\delta$, discount $\gamma_c$
\State Initialise main actor $\bar \pi_r^{\psi}$, backup critic $Q_b^{\theta}$
\For{each episode}
    \State $s \gets \overline{s}_{i},\; x \gets d - \delta$
    \While{not done}
        \State \textbf{Shield:} $\mu_{\text{safe}} \gets \Xi_{Q_b}(\bar \pi_r^{\psi})(s, x)$
        \State Sample $(a,y) \sim \mu_{\text{safe}}$
        \State Execute $(a,y)$, observe $(s',x')$, reward $r$, cost $c$
        \State Store transition $((s,x), (a, y), r, c, (s', r'))$
        \State $s \gets s'$, $r \gets r'$
    \EndWhile
    \State \textbf{Update:}
    \begin{itemize}
        \item Update $\pi_\theta$ using shielded distributions from memory
    \end{itemize}
\EndFor
\State \textbf{return} projected policy $\tilde\pi_r^{\psi}$ (via Thm.~\ref{thm:projection})
\end{algorithmic}
\end{algorithm}


The shield precedes sampling, so the exploration is budget-safe.
The choice of the RL update for \(\theta\) is open (e.g., PG, PPO, A2C). Additionally, the backup critic \(Q_b\) can be learned, either in parallel, off-policy, or even
pre-computed.

The previous results in Sec.~\ref{sec-PSRA} allows us to derive the following key safety guarantees for Algorithm~\ref{alg:prosh-correct}.
%
\begin{theorem}[Safety and near-Optimality]
    \begin{itemize}
        \item[(i)] At any step of the algorithm, the output policy $\bar \pi$ satisfies
    \[
    \mathcal{C}_{\bar{\mathcal{M}}}(\bar \pi) \leq r_0 + \frac{2 \gamma_c \Delta_b}{1-\gamma_c},\quad \Delta_c=||Q_b-Q_b^*||_\infty
    \]
        \item[(ii)] The algorithm is asymptotically optimal in deterministic environments as $\Delta_b \rightarrow 0$. More precisely, let $\mathcal{M}$ be a constrained deterministic environment modeled by a CMDP, and for any $\eta>0$ and $\Delta_b$ small enough, there exists $\bar \pi \in \overline{\Xi}_{sh}$ with cost at most $\gamma_c d+\mathcal{E}$ for $\mathcal{E}=\frac{2\Delta_b\gamma_c}{1-\gamma_c}$, 
        such that
    \[
    \mathcal{R}(\bar \pi) \geq \max_{\pi \in \Pi} \mathcal{R}(\pi) - \eta
    \]
    \end{itemize}
\end{theorem}


Since our guarantees involve the difference $\Delta_b$ and the assumption $\Delta_b \rightarrow 0$, we discuss how restrictive this assumption is.
\begin{remark}[On assumption $\boldsymbol{\Delta_b\!\to 0}$]
\leavevmode
We present several cases in which $\Delta_b\rightarrow 0$.

\begin{itemize}
    \item{Tabular update:}
      With exact Q-learning and sufficient exploration, the classical
      result of \cite{watkins1992q} gives
      $\Delta_b\!\xrightarrow{\text{a.s.}}\!0$.
      \item{Neural networks:}  Recent results       \citep{agarwal2020finite,xie2021bellman}
 show that over-parameterised deep Q-learning
      converges to the Bellman fixed point when the optimiser drives
      training error to zero and the network remains within a
      neighbourhood of its initialization.
      \item{Batch (fitted) Q-evaluation:}  
      In the agnostic setting, fitted Q-iteration with
      minimax regression loss enjoys finite-sample
      $\mathcal{O}(1/\!\sqrt{n})$ guarantees
      \citep{liu2024information}, so $\Delta_b$ vanishes as the data
      set grows.
\end{itemize}
%
Hence, in all cases above, the safety slack
$\kappa(\Delta_b)=\gamma_c\Delta_b/(1-\gamma_c)$ converges to zero,
\textsc{ProSh} is safe up to a controlled vanishing coefficient, and approaches true constrained optimality.
\end{remark}

\subsection{\textsc{ProSh}-TD3} \label{sec:ProSh-TD3}

In this section we introduce \textsc{ProSh}-TD3, an implementation of ProSH based on TD3 \cite{fujimoto2018td3}, for continuous spaces. 
Hereafter we describe 
the four key components of \textsc{ProSh}-TD3:
the network architecture, the  implementation of the shield, the actor/critic pairs training, and the exploration strategy.

\textbf{Networks Architecture.}
The backup actor-critic pair follows the TD3 architecture. The backup critic has two output heads, $Q_b^{\theta_1}$ and $Q_b^{\theta_2}$, and takes as input a state-actor pair of the original environment $\mathcal M$. The backup actor $\pi_b^{\phi}$ is deterministic, and takes as input a state in $\mathcal M$ and outputs an action.
The main actor-critic pair also follows the TD3 architecture. The main critic also has two heads, $\overline{Q}_r^{\xi_1}$ and $\overline{Q}_r^{\xi_2}$, but takes as input a state-action pair of the augmented environment $\overline{\mathcal M}$. The main actor $\overline{\pi}^{\psi}_r$ takes as input a state in the augmented environment and outputs a flipping policy in the augmented environment. Furthermore, we also maintain \emph{target} networks $Q_b^{\theta^{targ}_1}$ and $Q_b^{\theta^{targ}_2}$, etc, which are Polyak averages of the original networks, that we use for computing training targets, in order to stabilize training.

\textbf{Shield Implementation.}
We implement the shield using the Shield-map 
in Def.~\ref{def:shield-map} as a fully differentiable layer after the output of the main actor. More precisely, if $\overline{s}=(s,x)$ and $\overline{\pi}^{\psi}_r(\overline{s})=\rho (a_1,y'_1)+ (1-\rho) (a_2,y'_2)$, we implement $\Xi$ as \[\Xi(\overline{\pi}^{\psi})(\overline{s})=(1-\lambda)\rho(a_1,y'_1)+(1-\lambda)(1-\rho) (a_2,y'_2)+\lambda\overline{a}_3\]
where $y'_1=\min(y_1,Q_b^{\theta_1}(s,a_1))$, $y'_2=\min(y_2,Q_b^{\theta_1}(s,a_2))$, $\overline{a}_3=\left(\overline{\pi}^{\psi}_r(\overline{s}),Q_b^{\theta_1}(s,\pi^{\phi}_b(s)\right)$,
and \[\lambda = \text{clip}\left(\frac{\rho y'_1+(1-\rho) y'_2-x}{\text{relu}\left(\rho y'_1+(1-\rho) y'_2-Q_b^{\theta_1}(s,\pi_b(s))\right)+\eta},0,1\right)\]
where $\eta$ is chosen sufficiently small ($\sim10^-8$).

\textbf{Actor-critic Pairs Training.}
We train the actor-critic pairs with batch sampling from a replay buffer in which we store every transition made in the augmented CMDP. For the backup critic, we use the training target $y_b(c,s')$, proposed in \cite{HH20}, equal to       
\[
\begin{aligned}
c+\gamma_c  \Bigg[& \beta \min_{j\in\{1,2\}} Q_b^{\theta_j^{targ}}\left(s',\pi_b^{\psi^{targ}}(s')+\epsilon\right) \\&+\frac{1-\beta}{2} \sum_{j\in\{1,2\}} Q_b^{\theta_j^{targ}}\left(s',\pi_b^{\psi^{targ}}(s')+\epsilon\right)\Bigg]
\end{aligned}
\]
where $\epsilon$ is a small gaussian noise and $\beta$ is a hyperparameter. Compared to the standard TD3 training target, this target helps combat the underestimation bias of TD3, improving the backup critic accuracy.
The backup actor and the main critic are updated with the standard TD3 training losses.
The main actor is updated taking into account the shield:
\[\mathcal L(\psi)=-\mathbb E_{\overline{s}\sim\mathcal D} \left[\sum_{k\in\{1,2,3\}} \lambda_k \overline{Q}_r^{\xi_1}\left(\overline{s},\overline{a}_k\right) \right], \text{ where }\]
\[\Xi(\overline{\pi}_r^{\psi})(\overline{s})=\sum_{k\in\{1,2,3\}}\lambda_k\overline{a}_k.\]

Note that the gradients can flow through $\lambda_k$ and $\overline{a}_k$ in the main actor training loss as they are computed in a fully differentiable way.

\textbf{Exploration.}
For a better estimation of the backup critic, we alternate between \emph{main} episodes and \emph{hybrid} episodes. In the \emph{main} episodes, we sample from the (shielded) main actor. In the \emph{hybrid} episodes, we explore using the shielded main actor until a certain step, after which we explore using the backup actor. All the actions of the backup actor are converted to risk-augmented actions by associating the risk of the current observed state. The corresponding backup actor is always $Q_b$-shielded by construction, so the sampling is safe during training.

\section{Experimental Evaluation}

\subsection{Experimental setup}

We implement \textsc{ProSh}-TD3 using the Stable-Baselines 3 implementation of TD3\footnote{https://github.com/DLR-RM/stable-baselines3}. We evaluate our implementation of \textsc{ProSh}-TD3 on six environments of the Safety Gymnasium benchmark suite: SafetyHalfCheetahVelocity, SafetyHopperVelocity, SafetyPointCircle, SafetyPointGoal, SafetyCarCircle, and SafetyCarGoal are from the Navigation Suite. More details on the environments are provided in the comments on the experimental results. We benchmark against diverse Constrained RL algorithms, two of them being Off-Policy Algorithms (TD3-Lagrangian and PID-Lagrangian), and three being On-Policy Algorithms (CPO, FOCOPS, PPO-Saute). Some algorithms are safer, such as PPO-Saute, while others impose less restrictions and seek higher rewards, such as TD3-Lagrangian. We included them despite not always being directly comparable -- PPO-Saute being in most cases the only algorithm that is safe during training -- to present the overall performances of \textsc{ProSh} against states of the art algorithms, both in terms of safety and 
rewards. On each environment, the algorithms are run using three different seeds, and the mean and standard deviation are computed. Finally, the algorithms were implemented in the Omnisafe infrastructure \cite{ji2023omnisafeinfrastructureacceleratingsafe} and use the hyperparameters provided by Omnisafe for each environment.




\subsection{Experimental Results}

\paragraph{Overall Performance.} Across all environments, \textsc{ProSh} maintains a high level of safety, with only occasional low-magnitude violations. PPO-Saute also achieves consistent safety, although this often coincides with reduced performance and lower rewards. The Off-Policy algorithms TD3-Lagrangian and PID-TD3 tend to exhibit unsafe behavior in most settings, while the On-Policy methods FOCOPS and CPO show less stable safety profiles, with costs that fluctuate even after extended training. Taken together, the results suggest that \textsc{ProSh} provides the best trade-off between safety and performance in the evaluated environments, especially when ensuring safety both at training and deployment time is a necessity.

\paragraph{The Velocity Suite} In 
SafetyHopperVelocity and SafetyHalfCheetahVelocity, the agent has to move as quickly as possible, while adhering to a velocity constraint. 

\begin{figure}[h]
\begin{center}
    \includegraphics[width=0.48\textwidth]{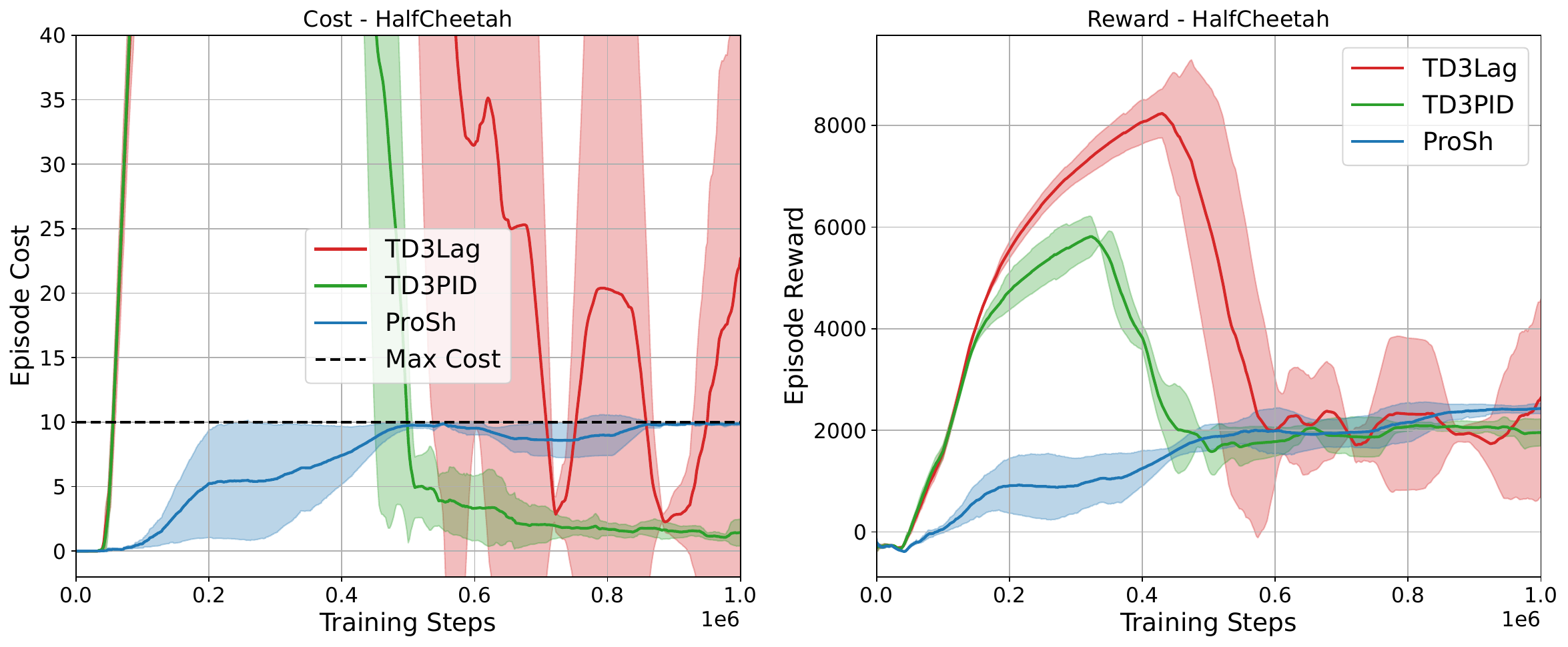}
    \caption{Comparison with Off-Policy Algorithms on Half Cheetah. Cost threshold $d=10$.}
\end{center}
\end{figure}
\begin{center}
    \includegraphics[width=0.48\textwidth]{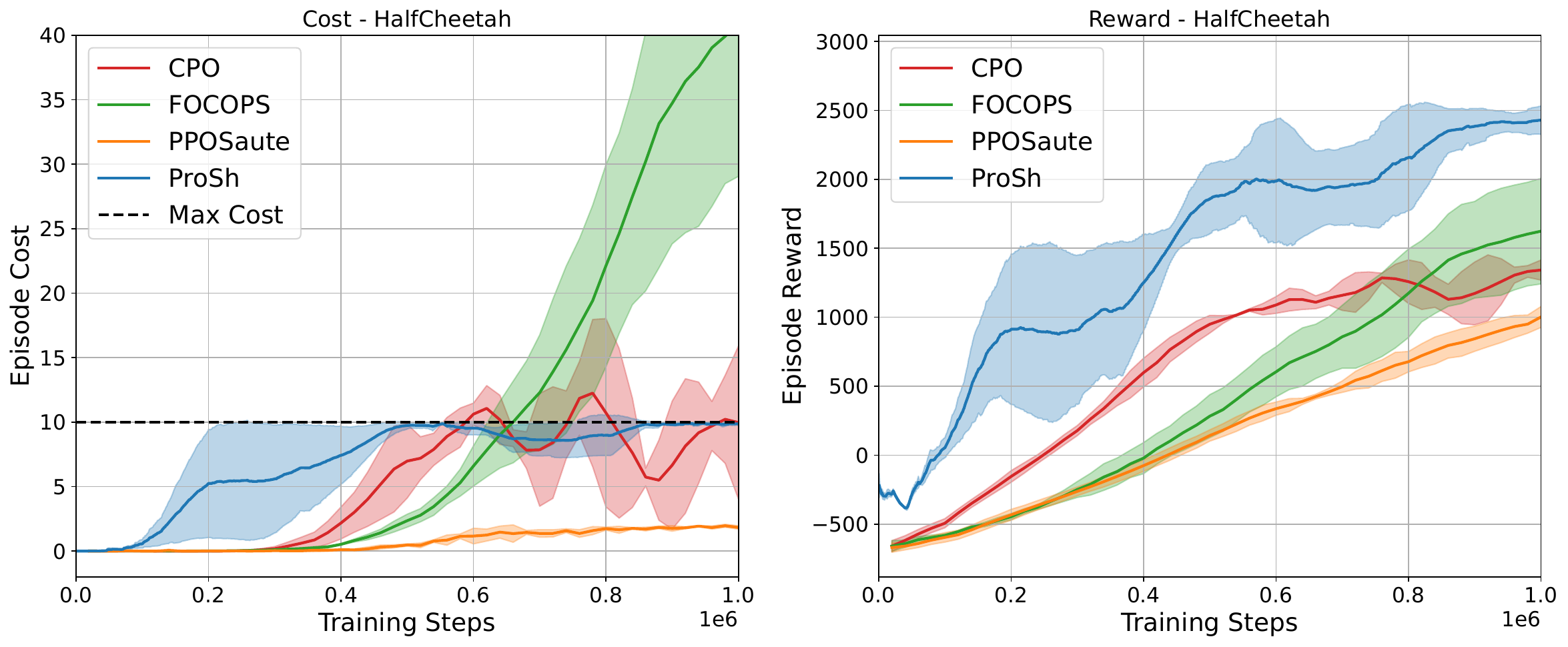}
\end{center}
\vspace{-0.5cm}
\begin{figure}[H]
    \caption{Comparison with On-Policy Algorithms on Half Cheetah. Cost threshold $d=10$.}
\end{figure}

\newpage

\textbf{Experimental Results: Safety Half Cheetah.} \textsc{ProSh} achieves absolute safety and stays under the cost constraint at all times. PID-TD3 displays comparable performance. TD3-Lag fails to be safe and does not achieve higher reward. \textsc{ProSh} displays higher reward than all the on-policy algorithms (CPO, FOCOPS, PPOSaute). Finally, FOCOPS fails to be safe and CPO violates the constraint after several steps.
\begin{center}
    \includegraphics[width=0.48\textwidth]{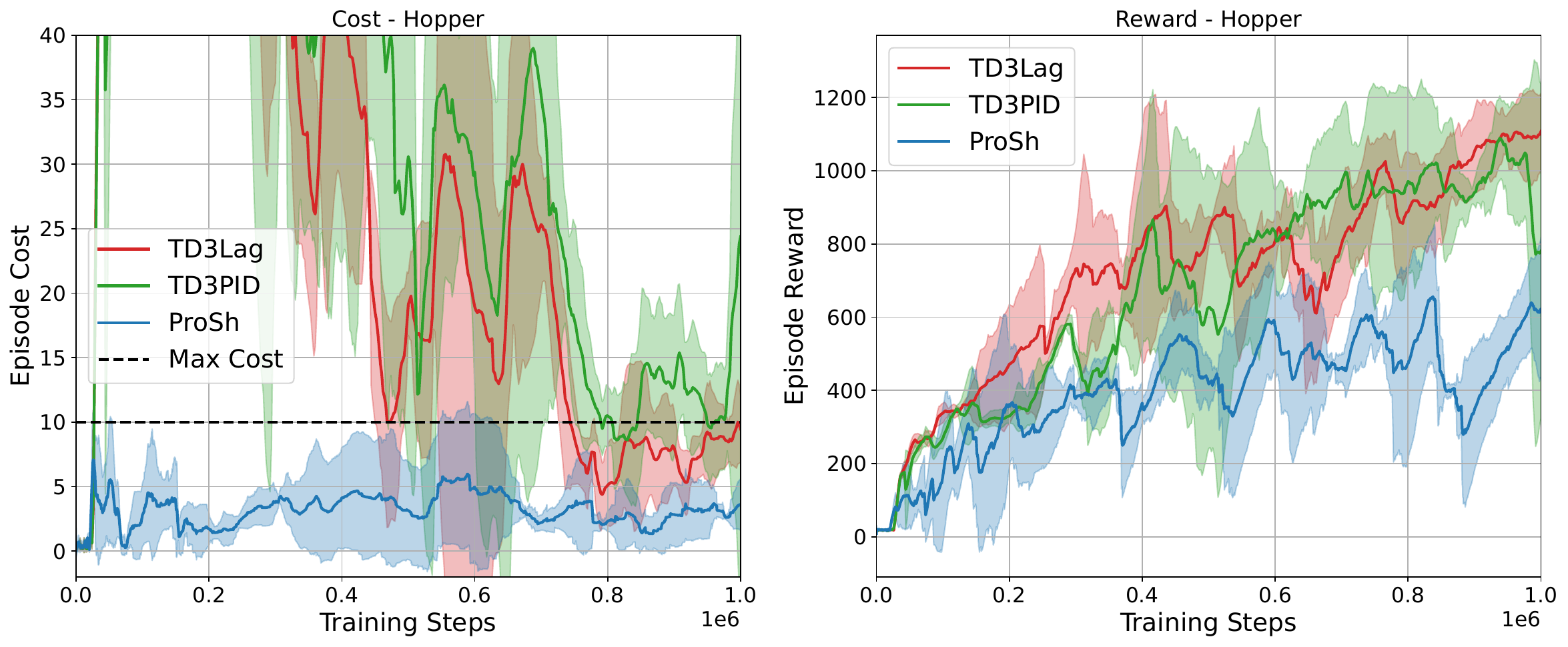}
\end{center}
\vspace{-0.5cm}
\begin{figure}[h]
    \caption{Comparison with Off-Policy Algorithms on Hopper. Cost threshold $d=10$.}
\end{figure}
\begin{center}
    \includegraphics[width=0.48\textwidth]{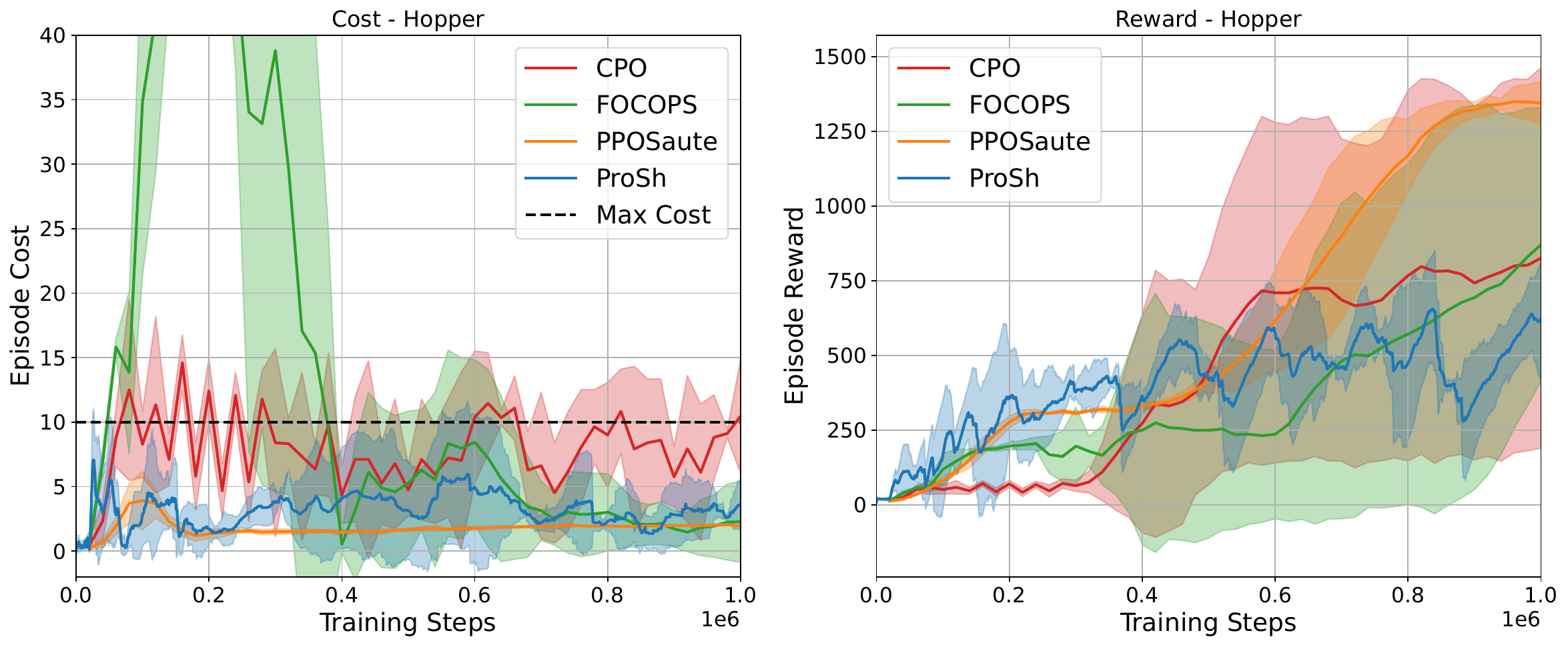}
\end{center}
\vspace{-0.5cm}
\begin{figure}[h]
    \caption{Comparison with On-Policy Algorithms on Hopper. Cost threshold $d=10$.}
\end{figure}

\textbf{Experimental Results: Safety Hopper.} \textsc{ProSh} achieves absolute safety and stays under the cost constraint at all times. The off-policy algorithms (TD3-Lag, PID-TD3) fail to be safe. FOCOPS and CPO also fail to be safe (including the standard deviation), while having comparable reward to \textsc{ProSh}. PPO-Saute is safe and has the highest reward on this environment.

\paragraph{The Navigation Suite} We consider four environments:
SafetyCarCircle, SafetyCarGoal, SafetyPointCircle, and SafetyPointGoal. The agent controls either a simple robot (Point), which is constrained to the 2D plane, with one actuator for turning and another for moving forward or backward; or a slightly more complicated robot (Car), which has two independently-driven parallel wheels and a free-rolling rear wheel that require coordination to properly navigate. Its goal is either to circle around the center of a circular area while avoiding going outside the boundaries (Circle), or to navigate to the Goal’s location while circumventing Hazards (Goal).

\begin{center}
    \includegraphics[width=0.48\textwidth]{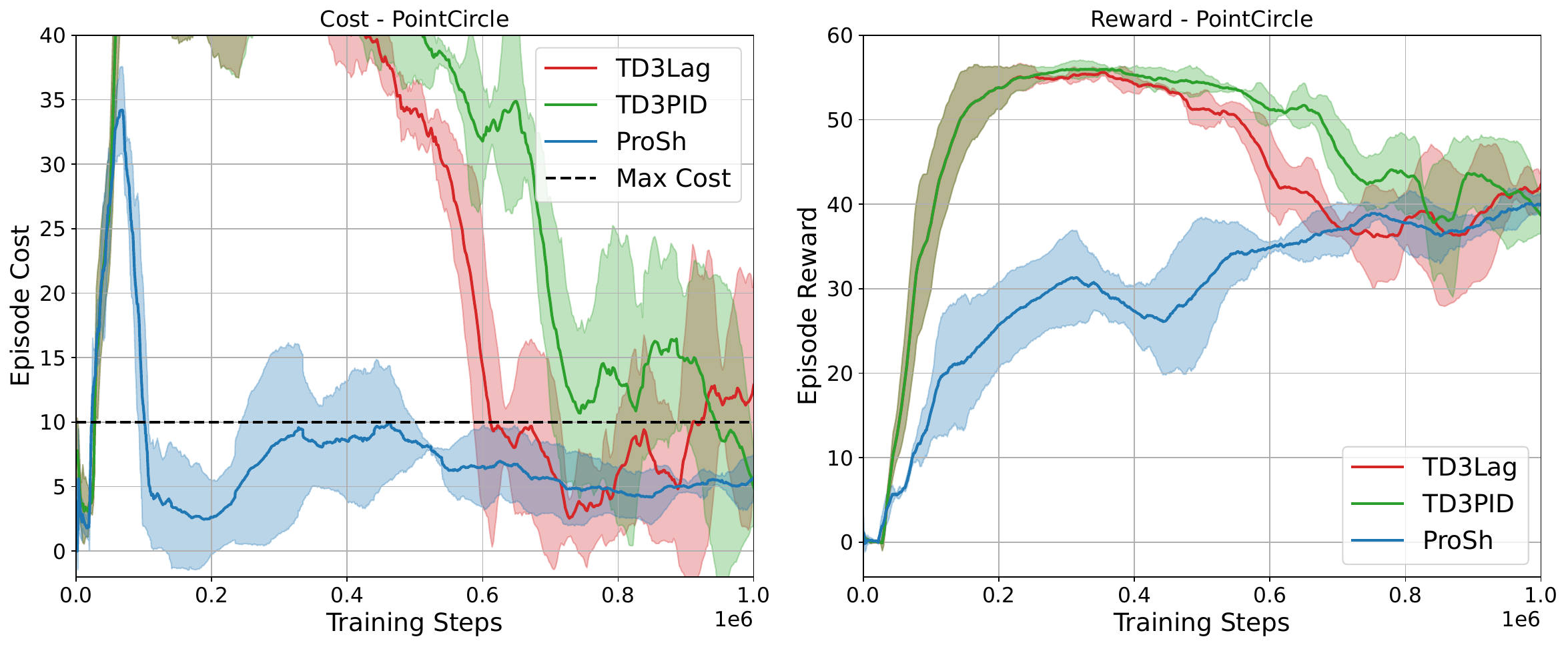}
\end{center}
\vspace{-0.5cm}
\begin{figure}[h]
    \caption{Comparison with Off-Policy Algorithms on Point Circle. Cost threshold $d=10$.}
\end{figure}
\begin{center}
    \includegraphics[width=0.48\textwidth]{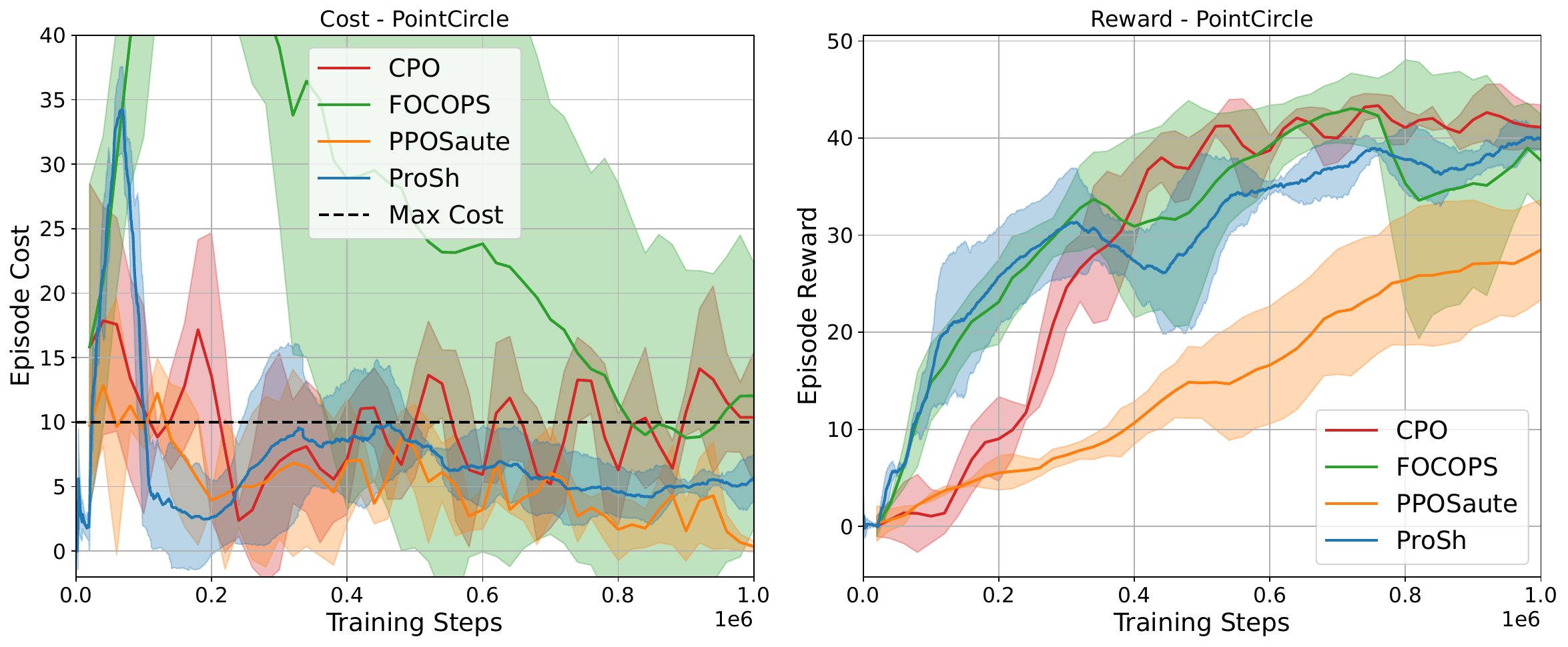}
\end{center}
\vspace{-0.5cm}
\begin{figure}[h]
    \caption{Comparison with On-Policy Algorithms on Point Circle. Cost threshold $d=10$.}
\end{figure}

\textbf{Experimental Results: Safety Point Circle.} \textsc{ProSh} achieves absolute safety and stays under the cost constraint at all times after a few steps. The off-policy algorithms (TD3-Lag, PID-TD3) fail to be safe and do not achieve higher rewards. FOCOPS and CPO also fail to be safe, while having comparable reward to \textsc{ProSh}. PPO-Saute is the only other safe algorithm, but it has lower performance.
\begin{center}
    \includegraphics[width=0.48\textwidth]{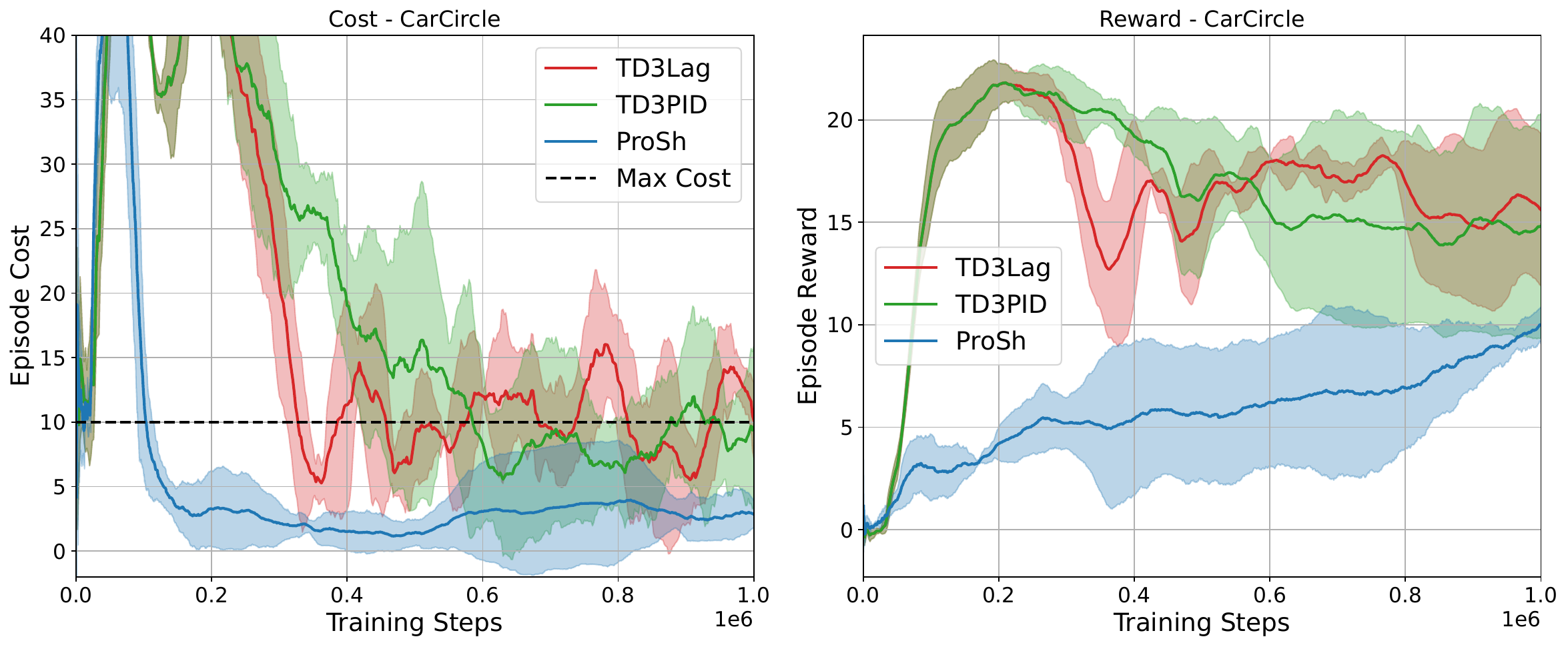}
\end{center}
\vspace{-0.5cm}
\begin{figure}[h]
    \caption{Comparison with Off-Policy Algorithms on Car Circle. Cost threshold $d=10$.}
\end{figure}
\begin{center}
    \includegraphics[width=0.48\textwidth]{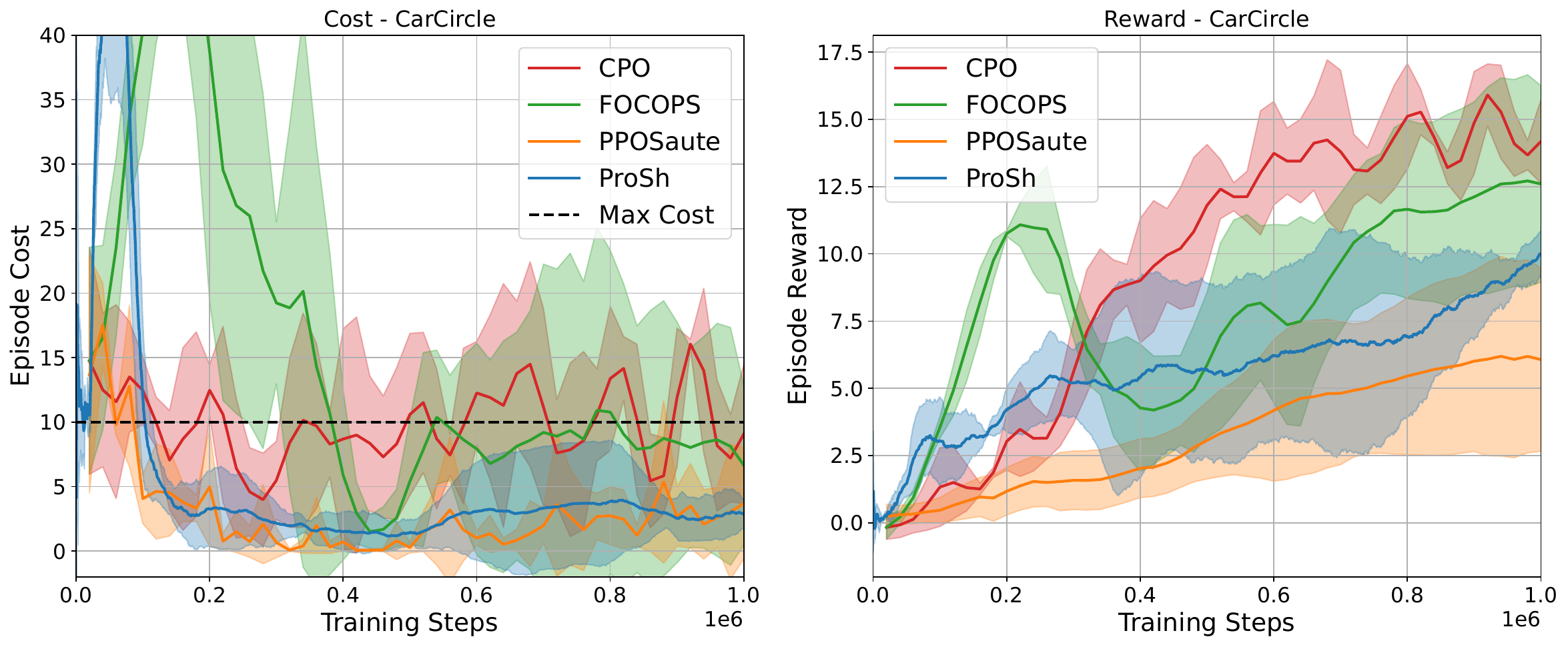}
\end{center}
\vspace{-0.5cm}
\begin{figure}[h]
    \caption{Comparison with On-Policy Algorithms on Car Circle. Cost threshold $d=10$.}
\end{figure}

\textbf{Experimental Results: Safety Car Circle.} \textsc{ProSh} achieves absolute safety and stays under the cost constraint at all times after a few steps. The off-policy algorithms (TD3-Lag, PID-TD3) are not safe. FOCOPS and CPO are also not safe and have very slightly greater rewards compared to \text{ProSh}. PPO-Saute is the only other safe algorithm, but has lower performance.

\begin{center}
    \includegraphics[width=0.48\textwidth]{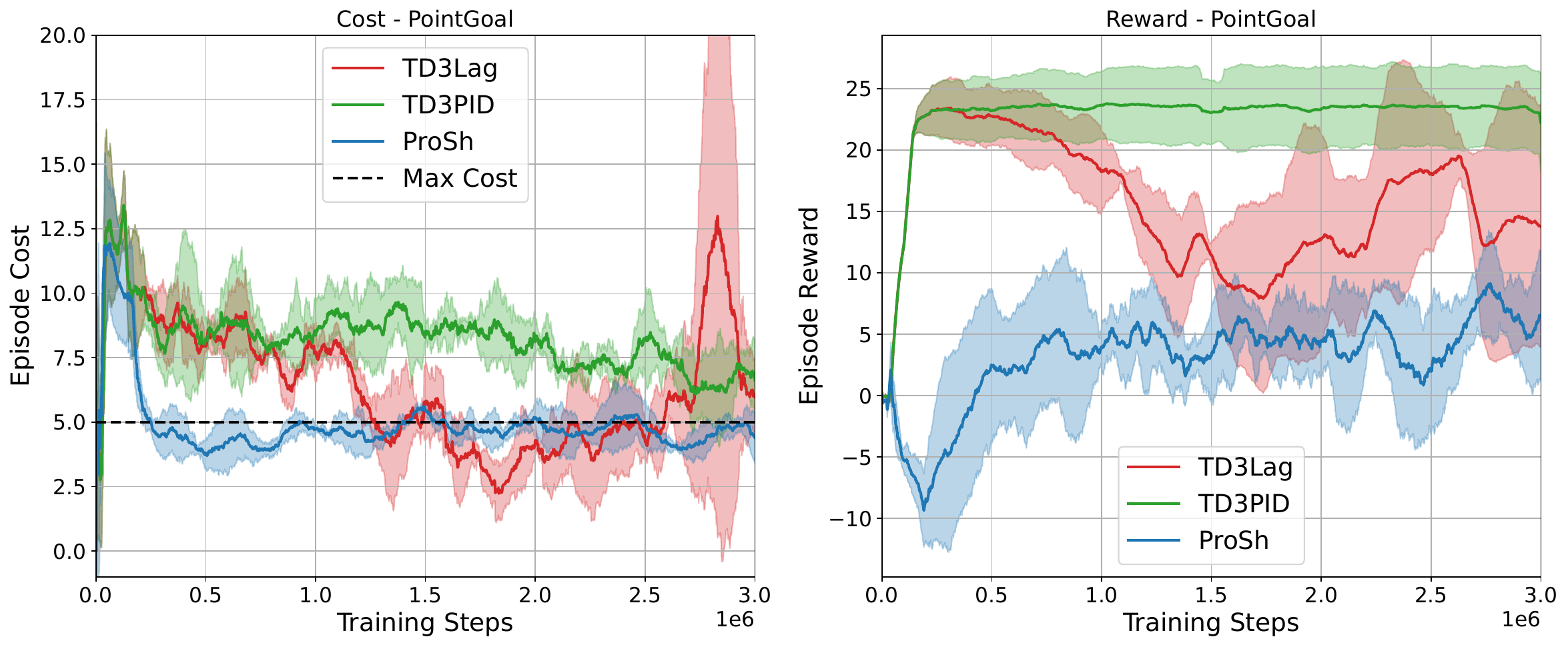}
\end{center}
\vspace{-0.5cm}
\begin{figure}[h]
    \caption{Comparison with Off-Policy Algorithms on Point Goal. Cost threshold $d=5$.}
\end{figure}
\begin{center}
    \includegraphics[width=0.48\textwidth]{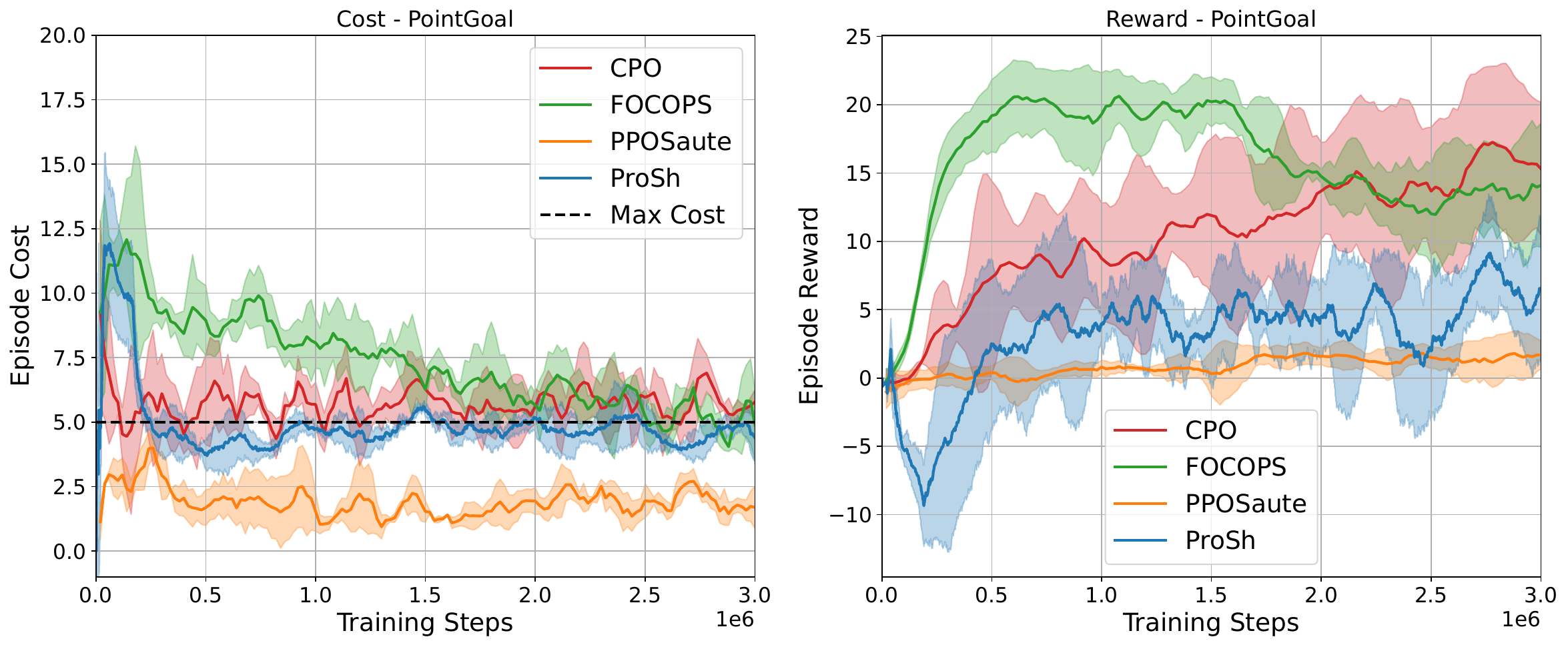}
\end{center}
\vspace{-0.5cm}
\begin{figure}[h]
    \caption{Comparison with On-Policy Algorithms on Point Goal. Cost threshold $d=5$.}
\end{figure}

\textbf{Experimental Results: Safety Point Goal.} Every algorithm is unsafe except PPO-Saute, that only achieves a tiny reward. \textsc{ProSh} is much safer than the other Algorithms, with a cost that evolves smoothly below the limit on average. Off-policy algorithms largely exceed the cost threshold, while on-policy algorithms stay closer to the threshold but violate the constraint even after several steps. In this environment, the higher safety of \textsc{ProSh} comes at the cost of diminished reward compared to CPO and FOCOPS.
\begin{center}
    \includegraphics[width=0.48\textwidth]{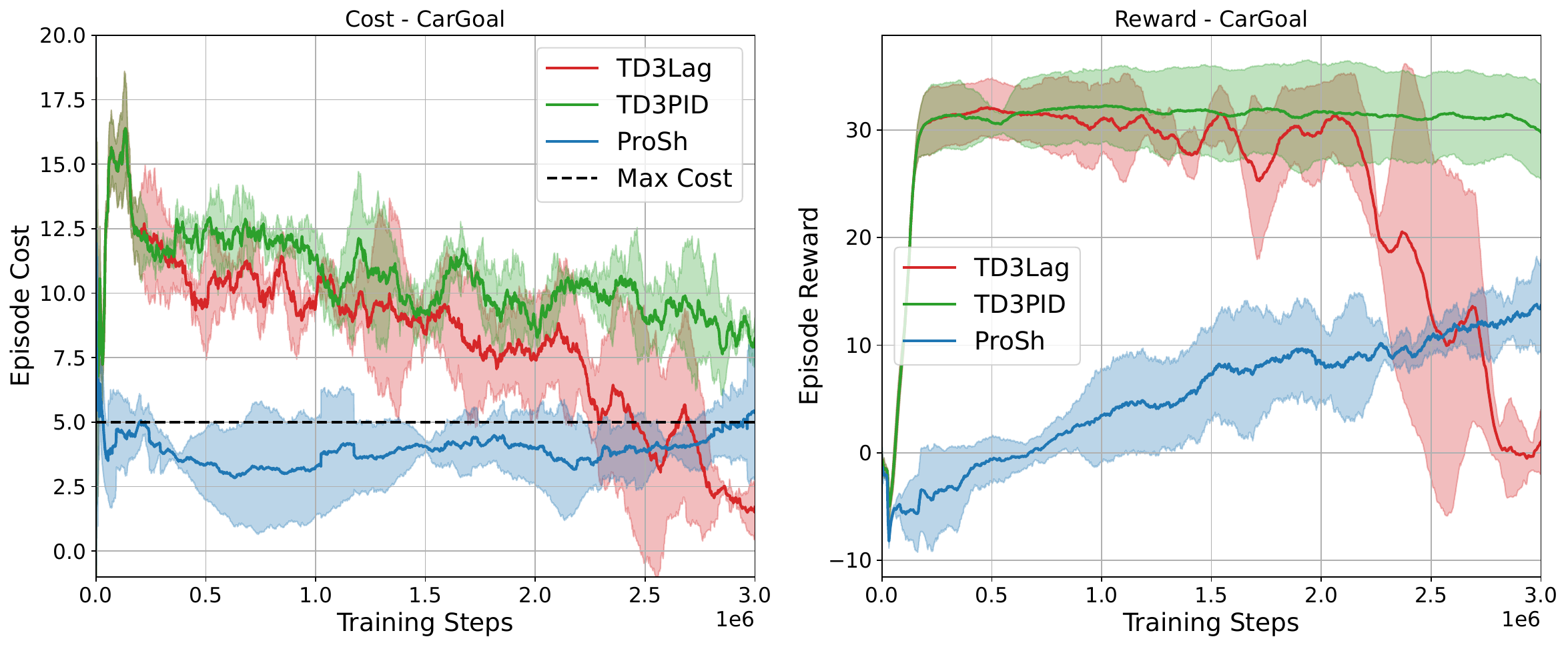}
\end{center}
\vspace{-0.5cm}
\begin{figure}[h]
    \caption{Comparison with Off-Policy Algorithms on Car Goal. Cost threshold $d=5$.}
\end{figure}
\begin{center}
    \includegraphics[width=0.48\textwidth]{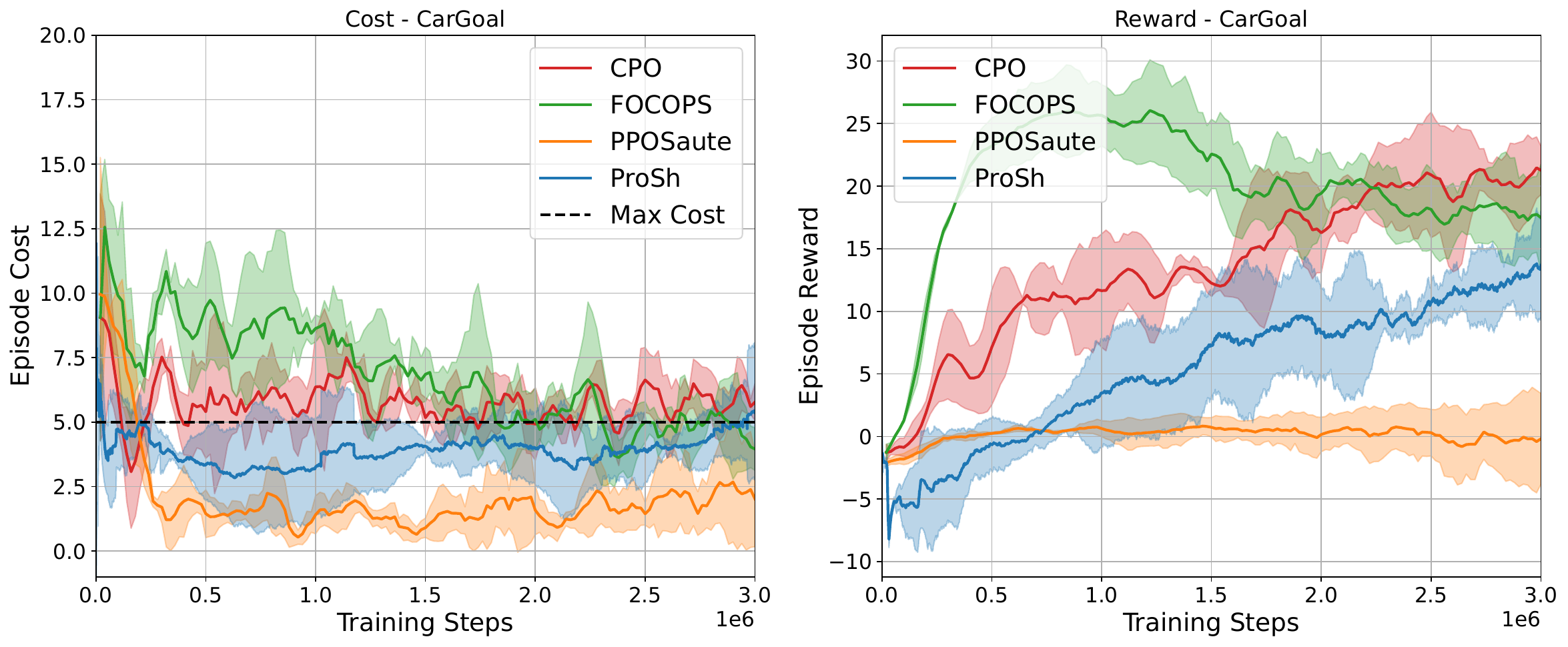}
\end{center}
\vspace{-0.5cm}
\begin{figure}[h]
    \caption{Comparison with On-Policy Algorithms on Car Goal. Cost threshold $d=5$.}
\end{figure}

\textbf{Experimental Results: Safety Car Goal.} PPO-Saute is always safe but only achieves a tiny reward. Among the others, \textsc{ProSh} is the safest algorithm, violating only the cost limit if the standard deviation is taken into account. The two off-policy algorithms TD3-Lag and TD3-PID are completely unsafe, while the two on-policy algorithms CPO and FOCOPS are unsafe even after many steps. Their rewards is slightly higher than \textsc{ProSh}, which showcases the best tradeoff between safety and reward.

\section{Conclusion}

We introduced \textsc{ProSh}, a model-free algorithm for safe reinforcement learning based on probabilistic shielding in a risk-augmented state space. Our method enforces strict safety during training, with formal guarantees depending only on the approximation quality of the backup critic $Q_b$. While this assumption is practically achievable in various settings, future work could leverage a dynamic risk margin that adapts to the critic's accuracy over time, ensuring safety even in early training stages. 

Our theoretical results focus on deterministic environments. To address the stochastic case, a \emph{risk-worthiness} function could be learned to guide risk allocation among successor states.

Finally, while our approach prioritized theoretical guarantees over fine-tuned optimization, the experimental results already demonstrate the soundness and promise of the core shielding mechanism. We expect performance to further improve with more refined implementations.

\paragraph{Acknowledgments.} The research described in this paper was partially supported by the EPSRC (grant number EP/X015823/1) and by the Moro-Barry family.

\newpage

\bibliographystyle{ACM-Reference-Format}   
\bibliography{bibliography}

\clearpage
\newpage
\onecolumn

\appendix

\section{Technical Appendices and Supplementary Material}
\subsection{ProSH-TD3 implementation}

We provide additional details on the implementation of \textsc{ProSH-TD3} and the theoretical guarantees that ensure its safety during training. First, we describe the full training procedure, including the update rules for both the main and backup actor–critic pairs. We then present the exact loss functions used to train each component of the algorithm. Finally, we prove that \textsc{ProSH-TD3} sampling remains safe even in the presence of exploration noise and hybrid training episodes.

\subsubsection{Pseudo-code and training}

The complete algorithm is outlined in Algorithm~\ref{alg:prosh-correct2}, including the training of the main actor and the backup actor, using normal episode where the main actor is used for sampling, and hybrid episode where the backup actor is used for sampling after several steps of the main actor.

\begin{algorithm}[H]
\caption{\textsc{ProSh}-TD3}
\label{alg:prosh-correct2}
\begin{algorithmic}[1]
\State \textbf{Input:} cost budget $d$, margin $\delta$, discount factors $\gamma,\gamma_c$, $hybrid\_delay$, $policy\_delay$
\State Initialize actor-critic pairs $Q_b^{\theta_i}, \pi_b^{\phi}, \overline{Q}_r^{\xi_i},\overline{\pi}_r^{\psi}$, $ep\_number \gets 0$, $L\gets []$
\For{each episode}
    \State $s \gets s_{\text{init}},\; r \gets d - \delta$, $step\_ind \gets 0$
    \If{$ep\_number\;\bmod\; hybrid\_delay = 0$}
        \State $switch\_step \gets \mathrm{rand}(0, \mathrm{mean}(L))$
    \Else
        \State $switch\_step \gets +\infty$
    \EndIf

    \While{not done}
    \If{$step\_ind \leq switch\_step$}
        \State Sample $\overline{a} \sim \Xi\left(\overline{\pi}_r^{\psi}\right)(s)$ and add noise to $\overline{a}$
        \State Execute $\overline{a}$, observe $\overline{s}'$, reward $r$, cost $c$
    \Else
        \State $(s,x) \gets \overline{s}$
        \State Sample $a \sim \pi_b^{\phi}$ and add noise to $a$
        \State Execute $\overline{a} = (a,x)$, observe $\overline{s}'$, reward $r$, cost $c$
    \EndIf
    \State store $(\overline{s}, \overline{a}, r, c, \overline{s}')$ in $\mathcal D$
    \State $\overline{s} \gets \overline{s}'$, $step\_ind\gets step\_ind+1$
    \State Sample from $\mathcal D$ and update critic networks by minimizing losses $\mathcal L(\theta_i)$, $\mathcal L(\xi_i) $
    \If{$ep\_number \;\bmod\; policy\_delay=0$}
        \State update actor networks by minimizing losses $\mathcal L(\phi)$, $\mathcal L(\psi) $
    \EndIf
    \EndWhile
    \State append $step\_ind$ to $L$, $ep\_number\gets ep\_number +1$
\EndFor
\State \textbf{return} projected policy $\tilde\pi$ (via Thm.~\ref{thm:projection})
\end{algorithmic}
\end{algorithm}

\paragraph{Actor-critic pairs training}
We train the actor-critic pairs with batch sampling from a replay buffer in which we store every transition made in the augmented MDP. For the backup critic, we use the training target $y_b(c,s')$, proposed in \cite{HH20}, equal to       
$$c+\gamma_c \left[\beta \min_{j\in\{1,2\}} Q_b^{\theta_j^{targ}}\left(s',\pi_b^{\psi^{targ}}(s')+\epsilon\right)+\frac{1-\beta}{2} \sum_{j\in\{1,2\}} Q_b^{\theta_j^{targ}}\left(s',\pi_b^{\psi^{targ}}(s')+\epsilon\right)\right],$$
where $\epsilon$ is a small gaussian noise and $\beta$ is a hyperparameter. Compared to the standard TD3 training target, this target helps combat the underestimation bias of TD3 \cite{HH20}.

Both of the backup critic heads are learned by regressing to this target:
    $$\mathcal L(\theta_i)=\mathbb E_{((s,x),(a,y),r,c(s',x'))\sim \mathcal D} \left[\left(Q_b^{\theta_i}(s,a)-y_b(c,s')\right)^2\right].$$
The backup actor is updated with the standard TD3 actor loss:
$$\mathcal L(\phi)=-\mathbb E_{(s,x)\sim \mathcal D} \left[Q_b^{\theta_1}(s,\pi^\phi(s))\right].$$
The main critic is updated with a loss adapted from TD3, but taking into account that the shielded main actor outputs three actions with their respective probabilities:
$$\mathcal L(\xi_i)=\mathbb E_{(\overline{s},\overline{a},r,c,\overline{s}')\sim \mathcal D} \left[\left(\overline{Q}_r^{\xi_i}(\overline{s},\overline{a})-y_r(r,\overline{s}')\right)^2\right], \text{ where }$$
$$y_r(r,s')= r+\gamma \sum_{k\in \{1,2,3\}} \lambda_k\min_{j\in\{1,2\}} \overline{Q}_r^{\theta_j^{targ}}\left(\overline{s}',\overline{a}_k+\epsilon\right)
\text{ with } \Xi(\overline{\pi}_r^{\psi})(\overline{s}')=\sum_{k\in\{1,2,3\}}\lambda_k\overline{a}_k.$$
Similarly, the main actor is updated taking into account the shield:
$$\mathcal L(\psi)=-\mathbb E_{\overline{s}\sim\mathcal D} \left[\sum_{k\in\{1,2,3\}} \lambda_k \overline{Q}_r^{\xi_1}\left(\overline{s},\overline{a}_k\right) \right] \text{ where } \Xi(\overline{\pi}_r^{\psi})(\overline{s})=\sum_{k\in\{1,2,3\}}\lambda_k\overline{a}_k.$$

Note that the gradients can flow through the $\lambda_k$ and the $\overline{a}_k$ in the main actor training loss as they are computed in a fully differentiable way.



\subsubsection{Safe sampling guarantees}

The algorithm \ref{alg:prosh-correct2} includes noise for the exploration as well as exploration rounds to train the backup actor and the backup critic. We will present Theorem \ref{thm:noise} stating that the sampling is still safe and provide an upper-bound for the cost. 

\begin{theorem}[Safety of the shielding with noise]\label{thm:noise}
\label{safetyshieldingwnoise}
    We consider a MDP $\mathcal{M}$ and any two policies $\bar \pi_1$ and $\bar \pi_2$, and $\xi \in [0,1]$ a small number. Then, with $\bar \pi$ the policy defined as 
    \[
    \bar \pi(s) = (1-\xi) \bar \pi_1(s) + \xi \bar \pi_2(s)
    \]
    for all $s\in \mathcal{S}$, the discounted cost of the policy $\bar \pi$ satisfies
    \[
    C(\bar \pi) \leq C(\bar \pi_1) + \frac{\xi c_{\max}}{1-\gamma_c} \frac{1}{1-(1-\xi)\gamma_c}.
    \]
\end{theorem}

We now consider the effective policy used during sampling, and not only the policy that the algorithm ouptuts. We aim to show that it also satisfies the safety guarantees. In practice, \textsc{ProSH} alternates between two modes: standard training episodes using the main actor, and hybrib episodes where the backup actor is used after several steps of the main actor. This alternation is used to improve the training of the backup critic $Q_b$. In both cases, a small noise is added.

For every augmented state $(s,x)$, both the main actor policy and the hybrid policy, without noise, can be written as
\[
\bar \pi (s,x) = \lambda (s,x) \bar \pi_1  + (1-\lambda(s,x)) \bar \pi_b,
\]

where $\bar \pi_1$ is the main actor and $\bar \pi_b$. Hence, the policy $\bar \pi$ is also $Q_b$-shielded, and its cost satisfies the estimate
\[
C(\bar \pi) \leq x_0 + \frac{2\Delta_b}{1-\gamma_c}.
\]

We now consider the final policy used for sampling with noise, and define
\[
\bar \pi_0 = (1-\xi) \bar \pi + \xi \bar \pi_{noise}.
\]

Using theorem \ref{thm:noise}, we have the upperbounds:
\[
C(\bar \pi_0) \leq C(\bar \pi) + \frac{1}{1-\gamma_c} \frac{\xi c_{\max}}{1-(1-\xi)\gamma_c}\leq x_0 + \frac{2\Delta_b}{1-\gamma_c}+ \frac{1}{1-\gamma_c} \frac{\xi c_{\max}}{1-(1-\xi)\gamma_c} .
\]
So the policy used for sampling is also safe, and the term induced by the noise can be removed by choosing
\[
x_0 \leq d - \frac{1}{1-\gamma_c} \frac{\xi c_{\max}}{1-(1-\xi)\gamma_c} .
\]

\newpage 

\subsection{Velocity and Navigation environments}

We evaluate our method on environments from the Safe Velocity and Safe Navigation suites of Safety-Gymnasium. Specifically, we use HalfCheetah, and Hopper (velocity tasks), and CarGoal, CarCircle, PointGoal, and PointCircle (navigation tasks). 

\textbf{Velocity Tasks.}
These tasks focus on moving forward quickly while remaining within a safe speed limit. The reward is proportional to forward progress, and the cost is binary: it equals 1 if the agent’s velocity exceeds $50\%$ of its training-time maximum, and 0 otherwise.
\begin{itemize}
\item HalfCheetah: A 2D robot with 9 links and 8 joints learns to run forward via torque control.
\item Hopper: A one-legged robot learns to hop forward by coordinating joint torques.
\end{itemize}

\textbf{Navigation Tasks.}
These tasks require agents to reach goals or follow paths while avoiding unsafe regions. The reward is based on proximity to the goal or path adherence, and the cost equals 1 upon contact with obstacles or hazardous zones.

\begin{itemize}
\item CarGoal vs. PointGoal: Both tasks require reaching a fixed target location.
\begin{itemize} 
\item CarGoal uses a wheeled robot with non-holonomic constraints, simulating a simplified autonomous vehicle.
\item PointGoal uses a free-moving point mass, which can change direction instantaneously.
\end{itemize}
\item CarCircle vs. PointCircle: The goal is to follow a predefined circular trajectory.
\begin{itemize}
\item CarCircle again uses the wheeled car agent, requiring continuous steering control.
\item PointCircle uses the point mass agent, allowing for simpler motion planning but still subject to the same cost constraints.
\end{itemize}
\end{itemize}

\begin{figure}[!h]
    \begin{subfigure}{0.24\textwidth}
        \includegraphics[width=\textwidth]{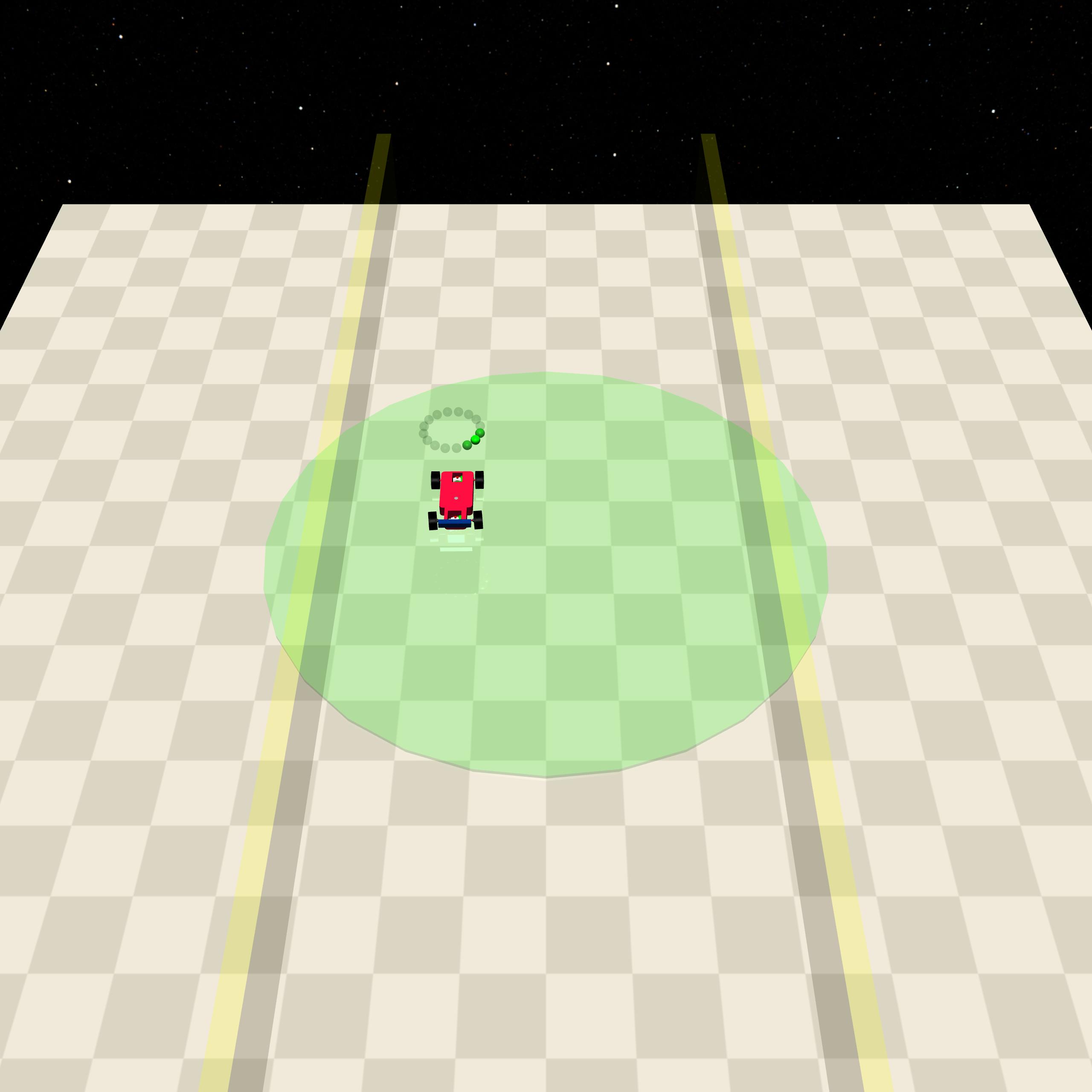}
        \caption{CarCircle}
    \end{subfigure}
    \begin{subfigure}{0.24\textwidth}
        \includegraphics[width=\textwidth]{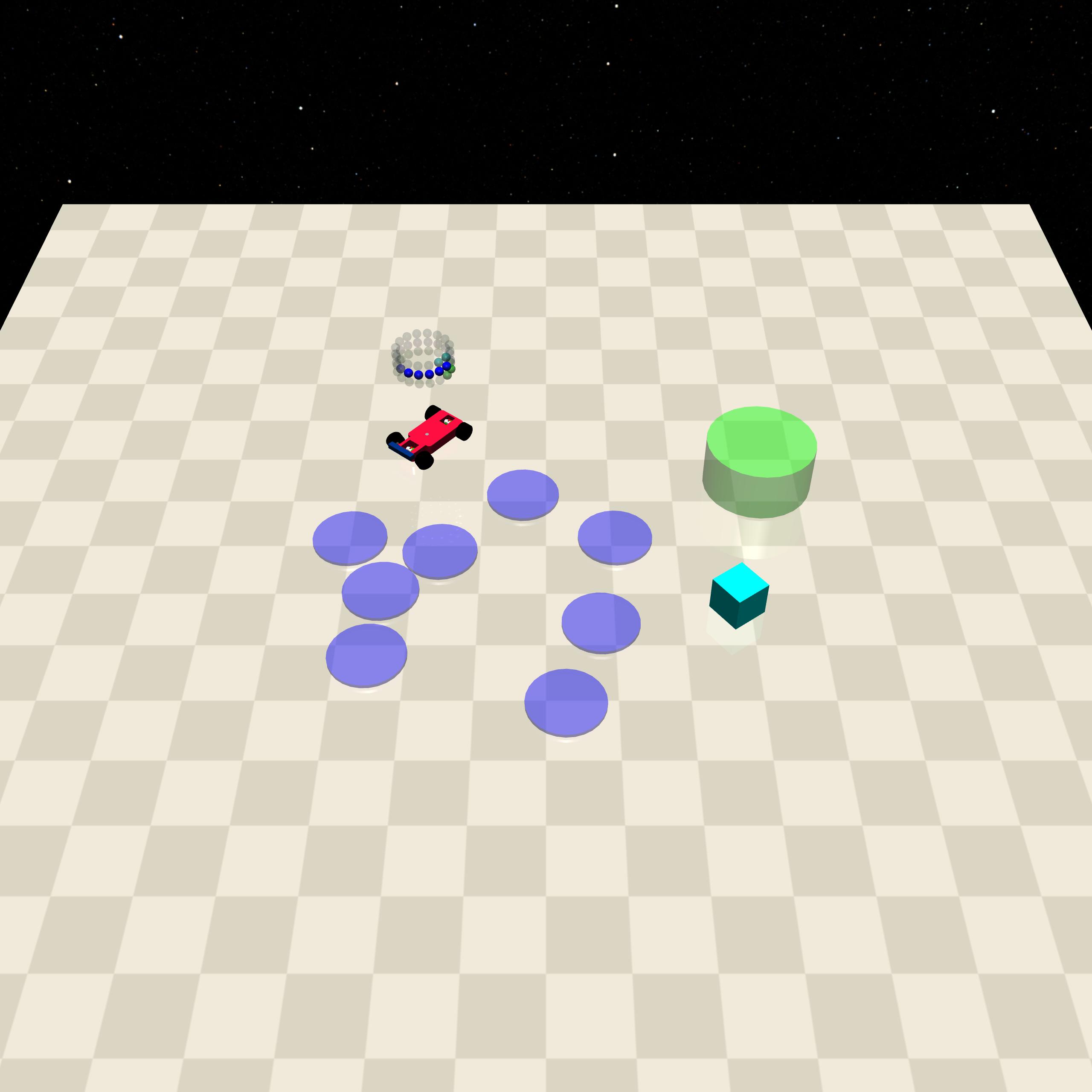}
        \caption{CarGoal}
    \end{subfigure}
    \begin{subfigure}{0.24\textwidth}
        \includegraphics[width=\textwidth]{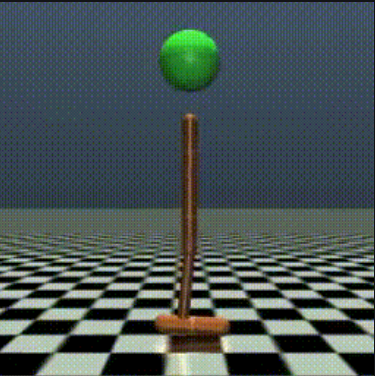}
        \caption{Hopper}
    \end{subfigure}
    \begin{subfigure}{0.24\textwidth}
        \includegraphics[width=\textwidth]{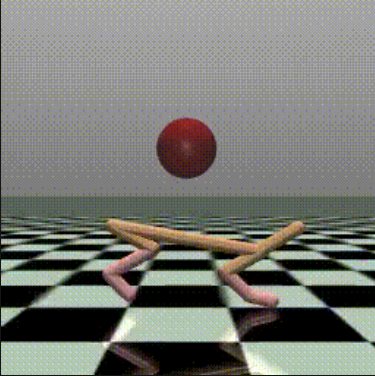}
        \caption{Half Cheetah}
    \end{subfigure}
    \caption{On the left, two environments of the Safe Navigation suite from Safety-Gymnasium, CarCircle and CarGoal. In CarCircle, the agent needs to circle around the center of the circle area while avoiding going outside the boundaries. In CarGoal, the agent needs to navigate to the Goal’s location while circumventing Hazards. In both case we choose $Level=1$. On the right, two environments of the Safe Velocity suite from Safety-Gymnasium, SafetyHopper and SafetyHalfCheetah. In both cases the agent has to move as quickly as possible while adhering to a velocity constraint. }
\end{figure}

\subsection{Experimental Setup}


\paragraph{Hyperparameters}
For PPO-Saute and TD3-Lag, we set $\gamma_c=0.995$, and we set $\gamma_{Saute}=0.995$ for PPO-Saute. For the other hyperparameters for TD3-Lag and PPO-Saute
, we use for each environments the default hyperparameters provided by Omnisafe \cite{ji2023omnisafeinfrastructureacceleratingsafe}. 
As for our algorithm, the set of hyperparameters used for every environment is given in Table \ref{table-hyperparam}. The more difficult the environment, the more the learning rate of the actors are decreased and the learning rate of the critics increased, as is standard for modern TD3 implementations.

\begin{table}[ht]
\begin{center}
\renewcommand{\arraystretch}{1.2}
\begin{tabular}{lcccccc}
\hline
\textbf{Hyperparameter} & \textbf{HalfCheetah} & \textbf{PointCircle} & \textbf{Hopper} & \textbf{CarCircle} & \textbf{PointGoal} & \textbf{CarGoal}\\
\hline
Main actor LR & 1e-4   & 1e-4   & 5e-4   & 5e-5  & 5e-5 & 5e-6  \\
Backup actor LR & 1e-4   & 1e-4   & 5e-4   & 5e-5 & 5e-6 & 5e-5 \\
Main critic LR & 1e-4   & 1e-4   & 1e-3   & 1e-3  & 1e-3 & 1e-4  \\
Backup critic LR & 1e-4   & 1e-4   & 1e-3   & 3e-4 & 3e-4 & 3e-4 \\
Reward Discount factor & 0.99   & 0.99   & 0.99   & 0.99 & 0.99 & 0.99 \\
Cost Discount factor & 0.995   & 0.995   & 0.995   & 0.995 & 0.995 & 0.995 \\
policy\_delay & 2   & 2   & 2   & 2 & 2 & 2\\
hybrid\_delay & 2   & 2   & 2   & 2 & 3 & 3\\
$\beta$ & 0.75   & 0.75   & 0.75   & 0.75 & 0.75 & 0.75\\
net\_arch & [400,300]   & [400,300]   & [400,300]   & [400,300] & [400,300] & [400,300] \\
Exploration Policy & $\mathcal{N}(0,0.2)$ & $\mathcal{N}(0,0.2)$ & $\mathcal{N}(0,0.2)$ & $\mathcal{N}(0,0.2)$ & $\mathcal{N}(0,0.2)$ &  $\mathcal{N}(0,0.2)$ \\
\hline
\end{tabular}
\end{center}
\caption{Hyperparameters for ProSH-TD3}\label{table-hyperparam}
\end{table}
In Table \ref{table-hyperparam}, $\beta$ is the hyperparameter in the backup critic target. 

\paragraph{Experimental Setup} 
All experiments were conducted on a PBS-managed high-performance computing cluster running Red Hat Enterprise Linux 8.5 (Ootpa). Each compute node is equipped with 2× AMD EPYC 7742 CPUs (128 cores total), 1 TB of RAM, and 8× NVIDIA Quadro RTX 6000 GPUs. We launched each experiment seed as an independent PBS job, allocating 1 CPU core, 12GB of RAM, and 1 GPU per job. A typical training run (1 million environment steps) took approximately 4 hours.

For PPO-Saute, CPO, FOCOPS, TD3-PID, and TD3-Lag, we used \texttt{Python} 3.10.0, \texttt{Omnisafe} version 0.5.0, \texttt{Torch} version 2.7.0 and \texttt{Cuda} 11.8.

For ProSH-TD3, we used \texttt{Python} 3.10.0, \texttt{Stable baselines3} version 2.6.0, \texttt{Torch} version 2.7.0 and \texttt{Cuda} 11.8.

\newpage

\newpage

\subsection{Theoretical Analysis}

In this section, we provide all the proofs for the theorem we have stated throughout the article. 

\subsubsection{Proof of Theorem~\ref{thm:safe_nondet}}

We recall Theorem \ref{thm:safe_nondet}:

\textbf{Theorem 1.} (Safety bound for $Q_b$-shields.)

\emph{
Let $\mathcal{M}$ be a CMDP with cost discount factor $\gamma_c$, and
let $Q_b^{*}$ be its optimal state–action cost function. Assume that $Q_b$ is a $Q$-value, and let $\Delta_b = ||Q_b-Q_b^*||_{L^\infty(\mathcal{S}\mathcal{A})}$.  For any policy
$\bar\pi\in\overline\Pi_{\mathrm{val}}$ that is $Q_b$-shielded, the expected
discounted cost from the augmented state $(s_0,x_0)$ obeys}
\[
\mathcal{C}^{(s_0,x_0)}(\bar\pi)
\;\le\;
\frac{x_0}{\gamma_c}+\frac{2\Delta_b}{1-\gamma_c},
\qquad
\text{whenever } x_0\ge Q_b(s_0).
\]

    In order to prove theorem~\ref{thm:safe_nondet}, we will first prove the following lemma.
\begin{lemma}\label{lemma:induction}
    Let $C_n(s,x)=\gamma_c D_{n-1}(s,x)$, where $D_n$ is the expected discounted cost of the $n$ first steps of the policy $\bar \pi$. Then, 
    \[
    \left\{
    \begin{aligned}
        &C_{n}(s,x) \leq x + f_{n}(\Delta_b),~\text{when }x\geq Q_b(s), \\
        &C_{n}(s,x) \leq Q_b(s) + g_n(\Delta_b),~\text{when }x < Q_b(s)
    \end{aligned}
    \right.
    \]
    as long as the sequences $f_n(\Delta_b)$ and $g_n(\Delta_b)$ satisfy $f_1(\Delta_b) \geq \gamma_c \Delta_b$, $g_1(\Delta_b) \geq \gamma_c \Delta_b$ and for every $n\geq 2$,
    \[
    \left\{
        \begin{aligned}
            &f_n(\Delta_b) \geq 2 \gamma_c \Delta_b + \gamma_c f_{n-1},\\
            &g_n(\Delta_b) \geq \max( 2\gamma_c \Delta_b + \gamma_c g_{n-1}(\Delta_b), 3 \gamma_c \Delta_b + \gamma_c f_{n-1}(\Delta_b)  ). 
        \end{aligned}
    \right.
    \]

\end{lemma}
\begin{proof}

    For $\bar \pi$ a $Q_b$-shielded policy of the augmented CMDP, we define the sequence  $C_n(s,x)$ as (we omit the dependency on $\bar \pi$)
    \[
    \begin{aligned}
    C_n(s,x) =& (1-\lambda) \sum_{a \in A(s)} P_{\bar \pi}((a,y_a)|(s,x)) \sum_{s'} P(s',a) \gamma_c( c(s,a)+ C_{n-1}(s',x')) \\
    &+ \lambda \sum_{s'} P(s'|\pi_b(s)) \gamma_c ( c(s,\pi_b(s)) + C_{n-1}(s',Q_b(s') ),\quad C_0(s,x) = 0,
    \end{aligned}
    \]
    where $x'=y_a-Q_b(s,a)+Q_b(s')$. Note that $C_n(s,x)=\gamma_c D_{n-1}(s,x)$, where $D_{n-1}(s,x)$ is the average discounted expected cost of the policy $\bar \pi$. We used the fact that $P((s',x')|(a,y_a)) = P(s'|a)$ when $x'=y_a-Q_b(s,a)+Q_b(s')$. Finally, throughout this proof, when there is no ambiguity, we will denote $P_a = P((a,y_a)|(s,x))$. 

    We will show by induction that
    \[
    \left\{
    \begin{aligned}
    &C_{n}(s,x) \leq x + f_{n}(\Delta_b),~x\geq Q_b(s),\\
    &C_{n}(s,x) \leq Q_b(s) + g_{n}(\Delta_b),~x<Q_b(s).
    \end{aligned}
    \right.
    \]

    \textbf{Base case, $n=1$, $x\geq Q_b(s)$:}

    In this case, we have by definition that 
    \[
    \bar \pi(s,x) = (1-\lambda) \sum_{a\in A(s)} P_a (a,y_a) + \lambda (\pi_b(s),Q_b(s,\pi_b(s))),
    \]
    where 
    \[
    x \geq \gamma_c (1-\lambda) \sum_a P_a y_a + \gamma_c \lambda Q_b(s,\pi_b(s)).
    \]
    The cost function $C_1(s,x)$ satisfies by definition
    \[
    \begin{aligned}
    C_1(s,x) =& (1-\lambda) \sum_{a} P_a \sum_{s'} P(s'|a) \gamma_c (c(s,a)+C_0(s',x')) \\&+ \lambda \gamma_c \sum_{s'} P(s'|\pi_b(s)) (c(s,\pi_b(s))+C_0(s',y') ),
    \end{aligned}
    \]
    so, using $C_0=0$,
    \[
    C_1(s,x) = (1-\lambda) \sum_{a} P_a \sum_{s'} P(s'|a) \gamma_c c(s,a) + \lambda \gamma_c \sum_{s'} P(s'|\pi_b(s)) c(s,\pi_b(s)).
    \]
    Now, any policy taking action $a$ has a cost at least equal to $\sum_{s'} P(s'|a) c(s,a)$, so $Q_b^*(s,a)\geq \sum_{s'} P(s'|a) c(s,a)$. Similarly, $Q_b^*(s,\pi_b(s)) \geq \sum_{s'} P(s'|\pi_b(s)) c(s,\pi_b(s))$. Using $||Q_b^*-Q_b||_{\infty} \leq \Delta_b$ gives
    \[
    \sum_{s'} P(s',a) c(s,a) \leq Q_b(s,a) + \Delta_b,\quad \sum_{s'} P(s'|\pi_b(s)) c(s,\pi_b(s)) \leq Q_b(s,\pi_b(s)) + \Delta_b.
    \]
    Hence, we get 
    \[
    C_1(s,x) \leq (1-\lambda) \gamma_c \sum_a P_a Q_b(s,a) + \lambda \gamma_c Q_b(s,\pi_b(s)) + \gamma_c \Delta_b.
    \]
    Since for any $a\in A(s)$, $y_a\geq Q_b(s,a)$, we obtain
    \[
    C_1(s,x) \leq (1-\lambda) \gamma_c \sum_a P_a y_a + \lambda \gamma_c Q_b(s,\pi_b(s)) + \gamma_c \Delta_b. 
    \]
    Using the weight condition, we finally obtain
    \[
    C_1(s,x) \leq x + \gamma_c \Delta_b.
    \]

    \textbf{Base case, $n=1$, $x<Q_b(s)$:}

    In this situation, we have by the definition of a $Q_b$-shielded policy that
    \[
    \bar \pi (s,x) = (\pi_b(s),z),\quad \gamma_c z \leq x.
    \]
    In this case, the cost function $C_1$ satisfies
    \[
    C_1(s,x) = \sum_{s'} P(s'|\pi_b(s)) \gamma_c ( C_0(s',Q_b(s') )+c(s,\pi_b(s))) = \sum_{s'} P(s'|\pi_b(s)) \gamma_c c(s,\pi_b(s)) .
    \]
    Any policy taking $s'$ with probability $1$ would have a cost of at least $\sum_{s'} P(s'|\pi_b(s)) c(s,\pi_b(s)) $, which means 
    \[
    Q_b^*(s,\pi_b(s)) \geq \sum_{s'} P(s'|\pi_b(s)) c(s,\pi_b(s)), 
    \]
    so
    \[
    Q_b(s,\pi_b(s)) \geq \sum_{s'} P(s'|\pi_b(s)) c(s,\pi_b(s)) -\Delta_b. 
    \]
    Consequently, 
    \[
    C_1(s,x) \leq \gamma_c (Q_b(s,\pi_b(s)) + \Delta_b) = Q_b(s) + \gamma_c \Delta_b. 
    \]
    
    \textbf{ Induction step, $x\geq Q_b(s)$: }

    With the same notation, $\bar \pi$ now satisfies
    \[
    \bar \pi(s,x) = (1-\lambda) \sum_a P_a (a,y_a) + \lambda (\pi_b(s),Q_b(s,\pi_b(s))),
    \]
    where 
    \[
    x \geq \gamma_c (1-\lambda) \sum_a P_a y_a + \gamma_c \lambda Q_b(s,\pi_b(s)).
    \]
    By definition, the cost functional $C_n$ satisfies
    \[
    \begin{aligned}
    C_n(s,x)&= (1-\lambda) \sum_a P_a \sum_{s'} P(s'|a) \gamma_c (c(s,a) + C_{n-1}(s',x')) \\
    &+ \lambda \sum_{s'} P(s'|\pi_b(s)) \gamma_c (c(s,\pi_b(s))+C_{n-1}(s',Q_b(s'))),
    \end{aligned}
    \]
    where $x' = y_a - Q_b(s,a) + Q_b(s')$. Since for all $a\in A(s)$, $y_a \geq Q_b(s,a)$, then $x' \geq Q_b(s')$, so we can apply the induction hypothesis and find
    \[
    C_{n-1}(s',x') \leq x' + f_{n-1}(\Delta_b),\quad C_{n-1}(s',Q_b(s')) \leq Q_b(s') + f_{n-1}(\Delta_b).
    \]
    Hence, $C_{n}(s,x)$ satisfies
    \[
    C_n(s,x) \leq (1-\lambda) A + \lambda B + \gamma_c f_{n-1}(\Delta_b),
    \]
    where 
    \[
    A = \sum_a P_a \sum_{s'} P(s',a) \gamma_c (c(s,a) + x' ),
    \]
    \[
    B = \sum_{s'} P(s'|\pi_b(s)) \gamma_c (c(s,\pi_b(s)) + Q_b(s')).
    \]
    Now, we have 
    \[
    Q_b^*(s,a) = c(s,a) + \sum_{s'} P(s'|a) Q_b^*(s') \geq c(s,a) + \sum_{s'} P(s'|a) (Q_b(s') - \Delta_b),
    \]
    so $Q_b(s,a) \geq c(s,a) + \sum_a P_a (s'|a) - 2 \Delta_b$. Using additionally $x'=y_a - Q_b(s,a) + Q_b(s')$, we obtain
    \[
    \begin{aligned}
    A &\leq \sum_a P_a \sum_{s'} P(s'|a) \gamma_c (c(s,a) + y_a - Q_b(s,a) + Q_b(s')) \\
    &\leq \sum_a P_a \sum_{s'} P(s'|a) \gamma_c Q_b(s') - \sum_{a} P_a \sum_{s'} P(s'|a) \gamma_c Q_b(s') \\
    &+ \sum_{a}P_a \sum_{s'} P(s',a) \gamma_c y_a +  2 \gamma_c \Delta_b \\
    & \leq \sum_a \gamma_c P_a y_a + 2 \gamma_c \Delta_b.
    \end{aligned}
    \]
    We reason in a similar way for the second term $B$. We have 
    \[
    Q_b^*(s,\pi_b(s)) \geq c(s,\pi_b(s)) + \sum_{s'} P(s'|\pi_b(s)) (Q_b(s')-\Delta_b),
    \]
    so $Q_b(s,\pi_b(s)) \geq c(s,\pi_b(s))+  \sum_{s'} P(s'|\pi_b(s)) Q_b(s')- 2 \Delta_b$. Using this inequality gives for $B$
    \[
    B \leq \sum_{s'} P(s'|\pi_b(s)) \gamma_c (c(s,\pi_b(s))+ Q_b(s')) \leq \gamma_c Q_b(s,\pi_b(s)) + 2 \gamma_c \Delta_b.
    \]
    Finally, we get 
    \[
    C_n(s,x) \leq (1-\lambda) \sum_a \gamma_c P_a y_a + (1-\lambda) 2 \gamma_c \Delta_b + \lambda \gamma_c Q_b(s,\pi_b(s)) + 2 \lambda \gamma_c \Delta_b + \gamma_c f_{n-1}(\Delta_b).
    \]
    Using the shielding condition on the weights that $\bar \pi$ satisfies, we finally obtain
    \[
    C_n(s,x) \leq x + 2 \gamma_c \Delta_b + \gamma_c f_{n-1}(\Delta_b) = x + f_{n}(\Delta_b).
    \]

    \textbf{Induction step, $x<Q_b(s)$:}
    In this case, $\bar \pi$ is given by
    \[
    \bar \pi (s,x) = (\pi_b(s),z),\quad \gamma_c z \leq x.
    \]
    This means that the cost functional satisfies 
    \[
    C_n(s,x) = \sum_{s'} P(s'|\pi_b(s)) \gamma_c ( c(s,\pi_b(s)) + C_{n-1}(s',x')),\quad x' = z - Q_b(s,\pi_b(s))+Q_b(s').
    \]
    Since $z\leq x / \gamma_c$, and $x < Q_b(s)$, we obtain $z < Q_b(s)/\gamma_c$. As a result,
    \[
    x' < Q_b(s)/\gamma_c - Q_b(s,\pi_b(s)) + Q_b(s').
    \]
    Now, $Q_b(s,\pi_b(s)) \geq Q_b^*(s,\pi_b(s)) - \Delta_b$, and $Q_b^*(s) \leq \gamma_c Q_b^*(s,\pi_b(s))$ as it is defined as the minimum over the actions, so 
    \[
    Q_b(s,\pi_b(s)) \geq Q_b^*(s)/\gamma_c - \Delta_b.
    \]
    Overall, 
    \[
    x' < Q_b(s') + \Delta_b.
    \]
    Two situations are possible here. Either $x' < Q_b(s')$, or $x' \geq Q_b(s')$. We start with $x'< Q_b(s')$.

    In this case, we obtain
    \[
    C_{n-1}(s',x') \leq Q_b(s') + f_{n-1}(\Delta_b),
    \]
    so we get for $C_n(s,x)$
    \[
    C_n(s,x) \leq \sum_{s'} P(s'|\pi_b(s)) \gamma_c (c(s,\pi_b(s)) + Q_b(s')) + \gamma_c f_{n-1}(\Delta_b).
    \]
    Now, $Q_b(s') \leq Q_b^*(s') + \Delta_b$, and $\sum_{s'} P(s',\pi_b(s)) \gamma_c (c(s,\pi_b(s)) + Q_b^*(s')) = \gamma_c Q_b^*(s,\pi_b(s))$, so
    \[
    C_n(s,x) \leq \gamma_c Q_b^*(s,\pi_b(s)) +\gamma_c \Delta_b + \gamma_c f_{n-1}(\Delta_b).
    \]
    Finally, using $Q_b^*(s,\pi_b(s)) \leq Q_b(s,\pi_b(s)) + \Delta_b$, 
    \[
    C_n(s,x) \leq \gamma_c Q_b(s,\pi_b(s)) + 2 \gamma_c +\gamma_c f_{n-1}(\Delta_b) =\leq Q_b(s) + 2 \gamma_c + \gamma_c f_{n-1}(\Delta_b) \leq Q_b(s) + f_{n}(\Delta_b).
    \]

    Now, in the other situation, we have $x' \geq Q_b(s')$. But we also have $x' \leq Q_b(s') + \Delta_b$. Hence, we can write
    \[
    C_{n-1}(s',x') \leq x' + f_{n-1}(\Delta_b) \leq Q_b(s') + \Delta_b + f_{n-1}(\Delta_b).
    \]
    So 
    \[
    \begin{aligned}
    C_n(s,x) \leq& \sum_{s'} P(s'|\pi_b(s)) \gamma_c (c(s,\pi_b(s)) + x' + f_{n-1}(\Delta_b)) \\
    &\leq \sum_{s'} P(s'|\pi_b(s)) \gamma_c (c(s,\pi_b(s)) 
    +  Q_b(s') +\Delta_b + f_{n-1}(\Delta_b))\\
    &\leq \sum_{s'} P(s'|\pi_b(s)) \gamma_c (c(s,\pi_b(s)) + Q_b^*(s') + 2 \Delta_b + f_{n-1}(\Delta_b)) \\
    &\leq \gamma_c Q_b^*(s,\pi_b(s)) + 2 \gamma_c \Delta_b + \gamma_c f_{n-1}(\Delta_b) \\
    &\leq \gamma_c Q_b(s,\pi_b(s)) + 3 \gamma_c \Delta_b + \gamma_c f_{n-1}(\Delta_b) = Q_b(s) + 3 \gamma_c \Delta_b + \gamma_c f_{n-1}(\Delta_b).
    \end{aligned}
    \]
    \end{proof}
    We can now go on with the proof of theorem \ref{thm:safe_nondet}.
    \begin{proof}
        Assume the assumptions of the theorem are satisfied. We use lemma \ref{lemma:induction} and obtain
        \[
    \left\{
    \begin{aligned}
        &C_{n}(s,x) \leq x + f_{n}(\Delta_b),~\text{when }x\geq Q_b(s), \\
        &C_{n}(s,x) \leq Q_b(s) + g_n(\Delta_b),~\text{when }x < Q_b(s)
    \end{aligned}
    \right.
    \]
    as long as the sequences $f_n(\Delta_b)$ and $g_n(\Delta_b)$ satisfy $f_1(\Delta_b) \geq \gamma_c \Delta_b$, $g_1(\Delta_b) \geq \gamma_c \Delta_b$ and for every $n\geq 2$,
    \[
    \left\{
        \begin{aligned}
            &f_n(\Delta_b) \geq 2 \gamma_c \Delta_b + \gamma_c f_{n-1},\\
            &g_n(\Delta_b) \geq \max( 2\gamma_c \Delta_b + \gamma_c g_{n-1}(\Delta_b), 3 \gamma_c \Delta_b + \gamma_c f_{n-1}(\Delta_b)  ). 
        \end{aligned}
    \right.
    \]
    We can choose $f_n(\Delta_b) = \sum_{k=1}^n 2 \Delta_b \gamma_c^k$ and $g_n(\Delta_b) = \sum_{k=1}^n 3 \Delta_b \gamma_c^k$, and we obtain for every $n$
        \[
    \left\{
    \begin{aligned}
        &C_{n}(s,x) \leq x + \sum_{k=1}^n 2 \Delta_b \gamma_c^k,~\text{when }x\geq Q_b(s), \\
        &C_{n}(s,x) \leq Q_b(s) + \sum_{k=1}^n 3 \Delta_b \gamma_c^k,~\text{when }x < Q_b(s)
    \end{aligned}
    \right.
    \]
    Subsequently, we obtain for the finite expected discounted cost 

        \[
    \left\{
    \begin{aligned}
        &D_{n-1}(s,x) \leq \frac{x}{\gamma_c} + \sum_{k=1}^n 2 \Delta_b \gamma_c^{k-1}\leq \frac{x}{\gamma_c} + 2 \Delta_b \frac{1}{1-\gamma_c},~\text{when }x\geq Q_b(s), \\
        &D_{n-1}(s,x) \leq \frac{Q_b(s)}{\gamma_c} + \sum_{k=1}^n 3 \Delta_b \gamma_c^{k-1}\leq \frac{Q_b(s)}{\gamma_c} + 3 \Delta_b \frac{1}{1-\gamma_c},~\text{when }x < Q_b(s).
    \end{aligned}
    \right.
    \]
    Since $D_n$ is Cauchy, we can just take the limit and obtain 
        \[
    \left\{
    \begin{aligned}
        &\mathcal{C}(s,x) \leq \frac{x}{\gamma_c} +  \frac{2 \Delta_b}{1-\gamma_c},~\text{when }x\geq Q_b(s), \\
        &\mathcal{C}(s,x) \leq \frac{Q_b(s)}{\gamma_c} +  \frac{3 \Delta_b}{1-\gamma_c},~\text{when }x < Q_b(s).
    \end{aligned}
    \right.
    \]
        
    \end{proof}
    
    \subsubsection{Proof of theorem \ref{thm:optimality}}

    We recall Theorem \ref{thm:optimality}.

    \textbf{Theorem 3.} (Optimality of the shielded policies)
    \emph{
    Let $\mathcal{M}$ be a deterministic CMDP with safety threshold $d$ and initial state $s_i$,
    $Q_b$ be a Q-value,
    $\Delta_b=||Q_b-Q_b^*||_{L^\infty(\mathcal{S}\mathcal{A})}$. Then, if we let $\Pi$ the set of all policies of $\mathcal{M}$, $\overline \Pi_{sh}$ be the set of shielded policies of $\overline{\mathcal{M}}$, and $\overline \Pi^f$ be the set of valued flipping policies, we have}
    \[
        \max_{\pi \in \Pi,~\mathcal{C}(\pi)\leq d} \mathcal{R}(\pi) \leq \max_{\bar \pi \in \Xi(\overline \Pi^f),~x_0\leq \gamma_c d+\rr} \mathcal{R}^{(s_i,x_0)}(\bar \pi) \leq \max_{\pi \in \overline \Pi_{sh},~x_0\leq \gamma_c d+\rr} \mathcal{R}^{(s_i,x_0)}(\bar \pi),
    \]
    \emph{for} $\rr=  \frac{2\Delta_b\gamma_c}{1-\gamma_c}$.

    Before proving Theorem \ref{thm:optimality}, we first need to introduce the following lemma.
    
    \begin{lemma}\label{lem:polytyope}
        For any deterministic MDP $\mathcal M$ and safety threshold $d$, there exists a $Q_b^\star$-shielded flipping memoryless policy $\overline{\pi}^*$ of $\overline{\mathcal M}^{Q_b^\star}$ such that 
        \[
        \mathcal{R}(\bar \pi^*) = \max_{\pi \in \Pi,~\mathcal{C}(\pi) \leq d} \mathcal{R}(\pi), 
        \]
        and starting at initial state $(s_0,d)$. Subsequently, $\bar \pi^*$ also satisfies 
        \[
        \mathcal{C}(\bar \pi^*) \leq d.
        \]
        
    \end{lemma}
    \begin{proof}
        Let $\mathcal M=(S,A,s_{i},P,R,C,d)$ be a deterministic CMDP with positive costs, and let $C_{max}$ be equal to $\frac{c_{max}}{1-\gamma_c}$. For any state $(s,x)$ of $\mathcal M$, and any mapping $y:A(s)\mapsto [0;C_{max}]$ such that $y(a)\geq Q_b^\star(s,a)$, we let $D(s,x,y)$ be the set of points $z\in \mathbb R^{A(s)}$ such that $x_a\geq 0$ for all $a\in A(s)$, \begin{equation} \sum_{a\in A(s)}z_a y(a)\leq \frac{x}{\gamma_c},\label{eq-hyp-1}
    \end{equation}
    and
    \begin{equation} \sum_{a\in A(s)}z_a= 1.\label{eq-hyp-2}
    \end{equation}
    For any $s,x,y$, $D(s,x,y)$ is the intersection between the convex polytope whose extreme points are all points $z^a\in \mathbb R^{A(s)}$ such that for any $a,a'\in A(s)$, $z^a_{a'}=1$ if $a=a'$ and $0$ otherwise, and the (convex) half-space defined by Equation (\ref{eq-hyp-1}). Thus, for any $s,x,y$, $D(s,x,y)$ is a convex polytope, and its extreme points are of the form $\lambda z^a+(1-\lambda) z^{a'}$ where $a,a'\in A(s)$. Let $V(s,x,y)$ be the extreme points of $D(x,s,y)$. We let $V^\star_b(s)=\min_{a\in A(s)} Q_b^\star(s,a)$, and we define $\hat{\mathcal M}$ to be the MDP with
\begin{itemize}
\item States : $\hat{S}= \left\{(s, x) \mid s \in S, x \in \left[V_b^\star(s);C_{\max}\right]\right\}$, with initial state $\overline{s}_{i}=(s_{i},d)$;
\item Actions : $\hat{A}(s,x)$ equal to the set of all $(y,v)$ where $y:A(s)\mapsto [0;C_{max}]$ and $y(a)\geq Q_b^\star(s,a)$ for any $a\in A(s)$ and $v\in V(s,x,y)$;
\item Rewards : $\hat{r}((s,x),(a,y))=r(s,a)$;
\item Transition Probability Function $\hat{P}$: for $\hat{s}=(s,x)$,  $\hat{s}'=(s',x')$ in $\hat{S}$, and $\hat{a}=(y,v)\in \hat{A}(\hat{s})$, we have 
\[\hat{P}\left( \hat{s},\hat{v},\hat{s}' \right) = 
    \begin{cases}
        0\text{ if } x'\neq y-Q^\star_b(s,a)+Q^\star_b(s') \text{ for all }a\in A(s)\\
        \sum_{a\mid x'= y-Q_b(s,a)+Q_b(s')}v_a P(s,a,s') \text{ otherwise .}
    \end{cases}
\]
\end{itemize}
First, notice that the MDP $\hat{\mathcal M}$ satisfies the hypothesis of the lower-continuous model (Definition 8.7 of \cite{BSSstochastic}), and thus admits an optimal memoryless policy $\hat{\pi}^\star$ (Corollary 9.17.2 of \cite{BSSstochastic}). Second, any memoryfull stochastic policy $\pi$ of $\mathcal M$ such that $\mathcal C(\pi)\leq d$ can be transformed into a memoryfull policy $\hat{\pi}$ of $\hat{\mathcal M}$ with the same cost and reward by letting $P_{\hat{\pi}}(s,a)=\sum_{v\in V(s,x,y)}\lambda_v (y,v)$, where $$y(a)=\mathbb E_{s_0,a_0,\ldots\sim \pi, (s_0,a_0)=(s,a)} \sum_{t} \gamma_c^t c(s_t,a_t),$$
and where the $\lambda_v$ are such that $\sum_{v\in V(s,x,y)}\lambda_v v=P_{\pi}$, $\lambda_v\in[0;1]$ and $\sum_{v\in V(s,x,y)} \lambda_v=1$. Conversely, similarly to Theorem \ref{thm:projection}, any memoryless deterministic policy of $\hat{\mathcal M}$ can be backwards projected to a memoryfull stochastic policy of $\mathcal M$ with the same cost and reward. Furthermore, by definition of $\hat{\mathcal M}$, any policy $\hat{\pi}$ of can be seen as a $Q_b^\star$-shielded policy of $\overline{\mathcal M}^{Q_b^\star}$, and we thus have by Theorem \ref{thm:safety} that $\mathcal C(\hat\pi)\leq d$. The lemma follows.
    \end{proof}

    \textbf{We now go on and provide a proof of Theorem \ref{thm:optimality}.}

    Using Lemma \ref{lem:polytyope}, we have the existence of a $Q_b^*$ policy $\bar \pi^*$ of the Risk-Augmented MDP $\overline{\mathcal{M}}^{Q_B^*}$, satisfying
    \[
    \left\{
    \begin{aligned}
        &\mathcal{R}(\bar \pi^*)= \max_{\pi \in \Pi,~\mathcal{C}(\pi)\leq d} \mathcal{R} (\pi),\\
        &x_0 = d,\\
        &\mathcal{C}(\bar \pi^*) \leq d.
    \end{aligned}
    \right.
    \]
    Moreover, $\bar \pi^*$ is $Q_b^*$-shielded, valued, and flipping, so we can write for any $(s,x)\in \overline{\mathcal{S}}$ such that for $x\geq Q_b^*(s)$,
    \[
    \bar \pi^*(s,x) = P_1 (a_1,y_1) + P_2 (a_2,y_2),
    \]
    where the risks satisfy
    \[
    \left\{
        \begin{aligned}
            &x = \gamma_c (P_1 y_1 + P_2 y_2),\\ 
            &y_1\geq Q_b^*(s,a_1),~y_2 \geq Q_b^*(s,a_2).
        \end{aligned}
    \right.
    \]
    Note that $\bar \pi^*$ is not defined yet on the set $\{(s,x),~x<Q_b^*(s) \}$. We show now that it does not matter, since this set can not be reached from the initial state.

    Define $D = \{ (s,x),~x \geq Q_b^*(s) \}$. Call $\overline{\mathcal{S}}_{\bar \pi^*,n}$ the set of states reached following $\bar \pi^*$ in $n$ steps with a non-zero probability. We have, $\overline{\mathcal{S}}_{\bar \pi^*,0} = \{ (s_0,d)\}\subseteq D$. By induction, assume now that $\overline{\mathcal{S}}_{\bar \pi^*,n}  \subseteq D$. For any $(s,x) \in D$, 
    \[
    \bar\pi^*(s,x) = P_1 (a_1,y_1) + P_2 (a_2,y_2),\quad y_j \geq Q_b^*(s,a_j).
    \]
    Hence, the state $\bar \pi^*$ can reach in one more step, using the definition of the augmented MDP, have risks satisfying
    \[
    x_j = y_j - Q_b^*(s,a_j) + Q_b^*(s_j) \geq Q_b^*(s_j),
    \]
    which means $\overline{\mathcal{S}}_{\bar \pi^*,n+1}  \subseteq D$.

    The challenge here is that $\bar \pi^*$ is not a $Q_b$-shielded policy, as $Q_b\neq Q_b^*$. We will now construct a policy $\bar \pi$, that will be $Q_b$-shielded, have a slightly increased cost, and yield the same reward.

    \textbf{Definition of $\bar \pi$.}

    The initial state of $\bar \pi$ is $(s_0,d+\rr)$. 
    
    We first define $\bar \pi$ on the set $\tilde D = \{ (s,x),~x \geq Q_b^*(s)+\rr \}$. We show later that it is enough. For $(s,x)\in \tilde D$, we define 
    \[
        \bar \pi(s,x) = P_1 (a_1,z_1) + P_2(a_2,z_2),\quad z_j = y_j + \rr + Q_b(s,a_j) - Q_b(s_j)-c(s,a_j),
    \]
    where $P_1$, $P_2$, $a_1$, $a_2$, $y_1$, $y_2$ are defined using the image of the optimal policy $\bar \pi_*$ at $(s,x-\rr)$ as in
    \[
        \bar \pi_*(s,x-\rr) = P_1 (a_1,y_1) + P_2 (a_2,y_2),
    \]
    and $s_1$ (resp. $s_2$) is the state reached when taking $a_1$ (resp. $a_2$) in $\mathcal{M}$, since $\mathcal{M}$ is assumed to be deterministic. Note that for $x\geq Q_b^*(s)+\rr$, we have $x-\rr\geq Q_b^*(s)$, so $\bar \pi_*$ is well defined.

   We first remark that the initial state $(s_0,\rr)$ is in $\tilde D$. Moreover, consider $\overline{S}_{n,\bar \pi}$, the set of states reachable by $\bar \pi$ in $n$ steps. We already have $\overline{S}_{0,\bar \pi} \subseteq \tilde D$. Assume now that $\overline{S}_{n-1,\bar \pi} \subseteq \tilde D$. We consider a path $\zeta = (s_0\dots s_n)$ of length $n$, and look at the last transition. Since $s_{n-1}\in \tilde D$, it is of the form $s_{n-1}=(s,x)$ $\bar \pi$ with $x \geq Q_b^*(s) + \rr$. $\bar \pi$ is defined as 
    \[
    \bar \pi(s,x) = P_1 (a_1,z_1) + P_2(a_2,z_2),\quad z_j = y_j + \rr + Q_b(s,a_j) - Q_b(s_j)-c(s,a_j), 
    \]
    where 
    \[
        \bar \pi_*(s,x-\rr) = P_1 (a_1,y_1) + P_2 (a_2,y_2).
    \]
    Using the definition of the augmented MDP for the Q-value $Q_b$, the state reached after taking the action $(a_j,z_j)$ is $(s_j,x_j)$ with
    \[
    x_j = z_j - Q_b(s,a_j) + Q_b(s_j) = y_j -c(s,a_j) + \rr.
    \]
    Since $y_j \geq Q_b^*(s,a_j) = c_j + Q_b^*(s_j)$ as $\bar \pi_*$ is $Q_b^*$-shielded, we have 
    \[
    x_j \geq Q_b(s_j) + \rr,
    \]
    so $s_n \in \tilde D$.

    \textbf{$\bar \pi$ is $Q_b$-shielded.}

        Now, we check if the policy is $Q_b$-shielded. Using the definition in the case where $x\geq Q_b(s)$, $\bar \pi$ is $Q_b$-shielded if and only if
    \begin{equation}\label{isqb}
    x \geq \gamma_c ( P_1 z_1 + P_2 z_2 ),\quad z_j = y_j + \rr + Q_b(s,a_j) - Q_b(s_j)-c(s,a_j).
    \end{equation}
    We use the fact that $\bar \pi^*$ is $Q_b^*$ shielded, so
    \[
    x-\rr \geq \gamma_c (P_1 y_1 + P_2 y_2).
    \]
    Plugging this in \eqref{isqb} gives
    \[
    \begin{aligned}
    \gamma_c ( P_1 y_1 + P_2 y_2  ) &+ \gamma_c ( P_1(-Q_b(s_1)+Q_b(s,a_1)-c(s,a_1)) \\
    &+ P_2 (-Q_b(s_2)+Q_b(s,a_2)-c(s,a_2)) ) +\gamma_c \rr \\
    \leq x-\rr +   \gamma_c &( P_1(-Q_b(s_1)+Q_b(s,a_1)-c(s,a_1)) \\
    &+ P_2 (-Q_b(s_2)+Q_b(s,a_2)-c(s,a_2)) ) +\gamma_c \rr \\
    \leq x + (\gamma_c-1) &\rr + 2 \gamma_c\Delta_b.
    \end{aligned}
    \]
    Here, we used that 
    \[
    \left\{
    \begin{aligned}
        &Q_b^*(s,a_j) = Q_b(s_j) + c(s,a_j),\\
        &|Q_b(s,a_j) - Q_b^*(s,a_j)| \leq \Delta_b,\\
        &|Q_b(s_j)-Q_b^*(s_j)| \leq \Delta_b,
    \end{aligned}
    \right.
    \]
    to obtain that $|-Q_b(s_2)+Q_b(s,a_2)-c(s,a_2))| \leq 2 \Delta_b$.
    
    Hence, $\bar \pi$ is $Q_b$-shielded when 
    \[
    2\gamma_c\Delta_b \leq (1-\gamma_c) \rr,
    \]
    so when $\rr \geq \frac{2\Delta_b\gamma_c}{1-\gamma_c}$.

    We also need to show that for everytime the policy $\bar \pi$ takes an action $(a,z)$, then $z \geq Q_b(s,a)$. We consider $(s,x)\in \tilde D$ and denote 
    \[
    \bar \pi(s,x) = P_1 (a_1,z_1) + P_2 (a_2,z_2).
    \]
    Since 
    \[
    z_j = y_j + Q_b(s_j) - Q_b(s,a_j) + c(s,a_j) + \rr,
    \]
    and 
    \[
    y_j \geq Q_b^*(s,a_j)
    \]
    since $\bar \pi_*$ is $Q_b^*$-shielded, we get 
    \[
        z_j \geq Q_b^*(s,a_j) - Q_b(s,a_j) + Q_b(s_j) + c(s,a_j) + \rr \geq Q_b(s_j) + c(s,a_j) -\Delta_b +\rr.
    \]
    Since $Q_b^*(s_j) + c(s,a_j) = Q_b^*(s,a_j)$ and $|Q_b(s_j)-Q_b^*(s_j)| \leq \Delta_b$, we finally obtain
    \[
    z_j \geq Q_b(s,a_j) - 2 \Delta_b + \rr.
    \]
    This is satisfied as long as $\frac{2\gamma_c}{1-\gamma_c} \geq 2$.
    
    \textbf{The rewards of $\bar \pi$ and $\bar \pi_*$ are equal.}

    We now define $\mathcal M_1$ and $\mathcal M_2$ the two Markov Chains as 
    \[
    \left\{
    \begin{aligned}
        &\mathcal{M}_1 = \overline{\mathcal{M}}^{Q_b^*}_{\bar \pi_*},\text{ the Markov Chained induced by $\bar \pi_*$ on } \overline{\mathcal{M}}^{Q_b^*},\\
        &\mathcal{M}_2 = \overline{\mathcal{M}}^{Q_b}_{\bar \pi},\text{ the Markov Chained induced by $\bar \pi$ on } \overline{\mathcal{M}}^{Q_b}.\\
    \end{aligned}
    \right.
    \]

    Since $\mathcal{M}$ is deterministic, $\overline{\mathcal{M}}^{Q_b^*}$ and $\overline{\mathcal{M}}^{Q_b}$ are also deterministic. $\mathcal{M}_1$ and $\mathcal{M}_2$ however are not deterministic, since the policies $\bar \pi^*$ and $\bar \pi$ are flipping. 

    \textbf{The Markov Chain $\mathcal{M}_1$:} Has one initial state, $(x_0,d)$. For any $(s,x)$ reachable by $\bar \pi_*$, $x \geq Q_b^*(s)$ so $\bar \pi_*$ is defined as 
    \[
    \bar \pi_* (s,x) = P_1 (a_1,y_1) + P_2 (a_2,y_2),
    \]
    for some $P_1$, $P_2$, $a_1$, $a_2$, $y_1$ and $y_2$. Call $s_1$ (resp. $s_2$) the state reached when taking $a_1$ (resp. $a_2$) in $\mathcal{M}$ from $s$. Then,
    \[
    P_{\bar \pi_*}((s_j,y_j-c_j)|(s,x)) = P_j,~c_j = c(s,a_j),
    \]
    and $P_{\bar \pi_*}(\bar z|(s,x)) =0$ for any other $\bar z$. Hence, the Markov Chain $\mathcal{M}_1$ has two transitions starting from $(s,x)$:
    \begin{itemize}
        \item $(s,x) \rightarrow (s_1,y_1-c_1)$ with probability $P_1$,
        \item $(s,x) \rightarrow (s_2,y_2-c_2)$ with probability $P_2$.
    \end{itemize}    

    \textbf{The Markov Chain $\mathcal{M}_2$:} Has one initial state, $(x_0,d+D)$. For any $(s,x)$ reachable by $\bar \pi_*$, $x \geq Q_b^*(s)$ so $\bar \pi_*$ is defined as 
    \[
    \bar \pi (s,x+D) = P_1 (a_1,z_1) + P_2 (a_2,z_2),~z_j = y_j-Q_b(s_j)+Q_b(s,a_j)+D,
    \]
    where $P_1$, $P_2$, $a_1$, $a_2$, $y_1$, $y_2$ are defined with $\bar \pi_*$ as 
    \[
    \bar \pi_*(s,x) = P_1 (a_1,y_1) + P_2 (a_2,y_2).
    \]
    and where $s_1$ (resp. $s_2$) is the state reached when taking $a_1$ (resp. $a_2$) in $\mathcal{M}$ from $s$. Then,
    \[
    P_{\bar \pi}((s_j,y_j-c_j+D)|(s,x)) = P_j,~c_j = c(s,a_j),
    \]
    and $P_{\bar \pi}(\bar z|(s,x)) =0$ for any other $\bar z$. Hence, the Markov Chain $\mathcal{M}_2$ has two transitions starting from $(s,x)$:
    \begin{itemize}
        \item $(s,x+D) \rightarrow (s_1,y_1-c_1+D)$ with probability $P_1$,
        \item $(s,x+D) \rightarrow (s_2,y_2-c_2+D)$ with probability $P_2$.
    \end{itemize}  

    With $\phi:\overline{\mathcal{S}}^{Q_b^*}\to \overline{\mathcal{S}}^{Q_b}$, defined as 
    \[
    \phi(s,x) = (s,x+D).
    \]
    We have that 
    \[
    \left\{
    \begin{aligned}
        &\phi(s_0,d) = (s_0,d+D),\\
        &P_{\mathcal{M}_1}( (s',x') | (s,x)) = P_{\mathcal{M}_2}( \phi(s',x') | \phi(s,x)),\\
        &r_{\mathcal{M}_1}((s,x)\rightarrow (s',x')) = r_{\mathcal{M}_2}(\phi(s,x)\rightarrow \phi(s',x')).
    \end{aligned}
    \right.
    \]

    Hence, we have for the expected discounted reward, as it is defined as an expectation, that 
    \[
    \mathcal{R}(\bar \pi) = \mathcal{R}(\bar \pi_*).
    \]

    \subsubsection{Proof of Remark \ref{rem:detopt}: better safety bound in the deterministic case.} 
    
    \begin{remark}\label{rem:detopt}
        In the deterministic case, one can choose instead, in the definition of the Risk-Augmented MDP, to take the risk $x_i = y_i - c(s,a)$ instead of the risk $x_i = y_i - Q_b(s,a) + Q_b(s')$ after taking the action $(a,y_i)$ and reaching $s'$. This is what we used in practice in that case, as the performances are slightly better. In that case, the optimality bound is improved, and we have
        \[
        \max_{\pi \in \Pi,~\mathcal{C}(\pi) \leq d} \mathcal{R}^{s_0}(\pi) \leq \max_{\bar \pi \in \overline{\Pi},~x_0 \leq d + 2 \Delta_b} \mathcal{R}^{(s_0,x_0)} (\bar \pi).
        \]
        The safety results are identical.
    \end{remark}

    \begin{proof}(Of Remark \ref{rem:detopt})

    The proof is very similar to the proof of \ref{thm:optimality}.     Using Lemma \ref{lem:polytyope} again, we have the existence of a $Q_b^*$ policy $\bar \pi^*$ of the Risk-Augmented MDP $\overline{\mathcal{M}}^{Q_B^*}$, satisfying
    \[
    \left\{
    \begin{aligned}
        &\mathcal{R}(\bar \pi^*)= \max_{\pi \in \Pi,~\mathcal{C}(\pi)\leq d} \mathcal{R} (\pi),\\
        &x_0 = d,\\
        &\mathcal{C}(\bar \pi^*) \leq d.
    \end{aligned}
    \right.
    \]
    Moreover, $\bar \pi^*$ is $Q_b^*$-shielded, valued, and flipping, so we can write for any $(s,x)\in \overline{\mathcal{S}}$ such that for $x\geq Q_b^*(s)$,
    \[
    \bar \pi^*(s,x) = P_1 (a_1,y_1) + P_2 (a_2,y_2),
    \]
    where the risks satisfy
    \[
    \left\{
        \begin{aligned}
            &x = \gamma_c (P_1 y_1 + P_2 y_2),\\ 
            &y_1\geq Q_b^*(s,a_1),~y_2 \geq Q_b^*(s,a_2).
        \end{aligned}
    \right.
    \]

    \textbf{Definition of $\bar \pi$.}
    
    We define $\tilde D = \{ (s,x),~x\geq Q_b^*+2\Delta_b \}$, and define the policy $\bar \pi$ on the augmented MDP $\overline{\mathcal M}^{Q_b}$ as follows. 
    \begin{itemize}
    \item The starting state of $\bar \pi$ is $(s_0,d+2\Delta_b)$.  
    \item For every $(s,x)\in \overline{S}\cap \tilde D$,
    \[
    \bar \pi(s,x) = P_1(a_1,z_1) + P_2(a_2,z_2),~ z_j =y_j + 2 \Delta_b,
    \]
    where $P_1$, $P_2$, $a_1$, $a_2$, $y_1$, and $y_2$ are defined via the image of the optimal policy $\bar \pi_*$ as 
    \[
    \bar \pi_*(s,x-2\Delta_b) = P_1 (a_1,y_1) + P_2(a_2,y_2).
    \]
    \end{itemize}
    First, Note that again, by a quick induction, all the states reachable from the initial state belong to $\tilde D$. This means that $\bar \pi$ is well defined despite having been only defined on $\tilde D$. 

    \textbf{$\bar \pi$ is $Q_b$-shielded.}

    We verify now that $\bar \pi$ is $Q_b$-shielded. For every $(s,x)\in \tilde D$, $x \geq Q_b^*(s) + 2 \Delta_b \geq Q_b(s)$. So the shielding conditions become
    \begin{enumerate}
        \item For every $(a_j,z_j)$ that the policy takes, we have $z_j \geq Q_b(s,a_j)$.
        \item The weights satisfy
        \[
        x \geq \gamma_c ( P_1 z_1 + P_2 z_2  ).
        \]
    \end{enumerate}
    For the first one, the optimal policy $\pi_*$ is $Q_b^*$ shielded, so that we have 
    \[
    y_j \geq Q_b^*(s,a_j).
    \]
    Hence,
    \[
    z_j \geq Q_b^*(s,a_j) + 2\Delta_b \geq Q_b(s,a_j) + \Delta_b.
    \]
    For the second one, we have, again because $\bar \pi_*$ is $Q_b^*$-shielded
    \[
    x-2\Delta_b \geq \gamma_c (P_1 y_1 + P_2 y_2),
    \]
    so
    \[
    \gamma_c(P_1 z_1 + P_2 z_2 ) = \gamma_c (P_1 y_1 + P_2 y_2 + 2\Delta_b) \leq (x - 2 \Delta_b) + 2 \gamma_c \Delta_b \leq x.
    \]

    So $\bar \pi$ is $Q_b$-shielded. Again, we prove that the rewards are identical by showing that the Markov Chain are identical, the details are identitcal to the proof of theorem \ref{thm:optimality}.
    
    \end{proof}

    \subsubsection{Proof of Theorem \ref{thm:noise}: Safety preserved with additional noise.}

    \textbf{Theorem 5.}(Safety of the shielding with noise)
    \emph{We consider a MDP $\mathcal{M}$ and any two policies $\bar \pi_1$ and $\bar \pi_2$, and $\xi \in [0,1]$ a small number. Then, with $\bar \pi$ the policy defined as }
    \[
    \bar \pi(s) = (1-\xi) \bar \pi_1(s) + \xi \bar \pi_2(s)
    \]
    \emph{for all $s\in \mathcal{S}$, the discounted cost of the policy $\bar \pi$ satisfies}
    \[
    C(\bar \pi) \leq C(\bar \pi_1) + \frac{\xi c_{\max}}{1-\gamma_c} \frac{1}{1-(1-\xi)\gamma_c}.
    \]

    \begin{proof}
    We let $C_n = \gamma_c D_{n-1}$, where $D_{n}$ is the discounted cost of the $n$ first steps. We will use an induction to show 
    \[
    C_{n}(\bar \pi,s) \leq C_{n}(\bar \pi_1,s) + \xi \frac{\gamma c_{\max}}{1-\gamma_c} \sum_{k=1}^n (1-\xi)^k \gamma_c^k.
    \]

    For $n=1$, we have 
    \[
    \begin{aligned}
    C_{1}(\bar \pi,s) =& (1-\xi) \sum_{a \in A(s)} P_{\bar \pi_1} (a|s) P(s'|a) \gamma_c (c(s,a) + C_0(\bar \pi,s')) \\
    &+ \xi \sum_{a \in A(s)} P_{\bar \pi_2} (a|s) P(s'|a) \gamma_c (c(s,a) + C_0(\bar \pi,s')) \leq (1-\xi) \gamma_c c_{\max}.
    \end{aligned}
    \]
    It is in particular smaller than the induction formula for $n=1$.

    Now, assuming the property holds for $k\leq n-1$. We write
    \[
    \begin{aligned}
    C_{n}(\bar \pi,s) =& (1-\xi) \sum_{a \in A(s)} P_{\bar \pi_1} (a|s) P(s'|a) \gamma_c (c(s,a) + C_{n-1}(\bar \pi,s')) \\
    &+ \xi \sum_{a \in A(s)} P_{\bar \pi_2} (a|s) P(s'|a) \gamma_c (c(s,a) + C_{n-1}(\bar \pi,s')),
    \end{aligned}
    \]
    so     
    \[
    \begin{aligned}
    C_{n} \leq & (1-\xi) \sum_{a \in A(s)} P_{\bar \pi_1} (a|s) P(s'|a) \gamma_c (c(s,a) + C_{n-1}(\bar \pi,s')) + \xi \gamma_c \frac{c_{\max}}{1-\gamma_c} \\
    \leq& (1-\xi) \sum_{a \in A(s)} P_{\bar \pi_1} (a|s) P(s'|a) \gamma_c \left( c(s,a) + C_{n-1}(\bar \pi_1,s') + \xi \frac{\gamma c_{\max}}{1-\gamma_c} \sum_{k=1}^n (1-\xi)^k \gamma_c^k \right) \\
    &+ \xi \gamma_c \frac{c_{\max}}{1-\gamma_c},
    \end{aligned}
    \]
    Now, we remark that 
    \[
    \sum_{a \in A(s)} P_{\bar \pi_1} (a|s) P(s'|a) \gamma_c \left( c(s,a) + C_{n-1}(\bar \pi_1,s') \right) = C_n(\bar \pi_1,s),
    \]
    so the inequality becomes
    \[
    \begin{aligned}
    C_{n}(\bar \pi,s)
        \leq& (1-\xi) C_n(\bar \pi_1,s) + (1-\xi) \xi \gamma_c \frac{\gamma_c c_{\max}}{1-\gamma_c} \sum_{k=1}^{n-1} (1-\xi)^k \gamma_c^k + \xi \gamma_c \frac{c_{\max}}{1-\gamma_c}\\
        \leq & C_n(\bar \pi_1,s) + \frac{\xi \gamma_c c_{\max}}{1-\gamma_c} \sum_{k=1}^n (1-\xi)^k \gamma_c^k,
    \end{aligned}
    \]
    which concludes the induction and the proof.
    
\end{proof}

\newpage

\end{document}